\documentclass[twoside,11pt]{article}
\usepackage{jmlr2e}
\usepackage{graphicx}
\usepackage{color}
\usepackage[table]{xcolor}
\usepackage{rotating}
\usepackage{tabularx}
\usepackage{pdflscape}
\usepackage{amsmath}
\usepackage{amsfonts}
\usepackage{algorithm}
\usepackage{algpseudocode}
\usepackage{bbm}
\usepackage{url}
\usepackage{hyperref}
\usepackage{pbox}
\usepackage{lscape}
\usepackage{dsfont}

\DeclareMathOperator*{\argmax}{arg\,max}
\newcommand{\ve}[1]{#1}
\newcolumntype{C}[1]{>{\centering}m{#1}}

\begin{document}

\title{A Survey on Contextual Multi-armed Bandits}

\author{\name Li Zhou \email lizhou@cs.cmu.edu \\
       \addr Computer Science Department\\
       Carnegie Mellon University\\
       5000 Forbes Avenue Pittsburgh, PA 15213, US
       }
\editor{}
\maketitle
\tableofcontents
\newpage

\section{Introduction}

\begin{table}[]
\hspace{-1.0cm}
\centering
\label{table:ruu}
\renewcommand{\arraystretch}{2}
\begin{tabular}{C{2.5cm}C{3.8cm}C{3.8cm}}
\cline{2-3}
\multicolumn{1}{C{2.5cm}|}{Learn model of outcomes} & \multicolumn{1}{C{3.8cm}|}{Multi-armed bandits} & \multicolumn{1}{C{3.8cm}|}{Reinforcement Learning} \\ \cline{2-3} 
\multicolumn{1}{C{2.5cm}|}{Given model of stochastic outcomes} & \multicolumn{1}{C{3.8cm}|}{Decision theory}     & \multicolumn{1}{C{3.8cm}|}{Markov Decision Process} \\ \cline{2-3} 
                                                        & Actions don't change state of the world  & Actions change state of the world           
\end{tabular}
\caption[Caption for LOF]{Four scenarios when reasoning under uncertainty.\footnotemark[1]}
\end{table}
In a decision making process, agents make decisions based on observations of
the world.
Table \ref{table:ruu}\footnotetext[1]{Table from CMU Graduate AI course
  slides. \url{http://www.cs.cmu.edu/~15780/lec/10-Prob-start-mdp.pdf}}
describes four scenarios when making decisions under uncertainty. In a
multi-armed bandits problem, the model of outcomes is unknown, and the outcomes can
be stochastic or adversarial; Besides, actions taken won't change the state of
the world.

In this survey we focus on multi-armed bandits. In this problem the agent needs to make a
sequence of decisions in time $1, 2, ..., T$. At each time $t$ the agent is given a
set of $K$ arms, and it has to decide which arm to pull. After pulling
an arm, it receives a reward of that arm, and the rewards of other arms are
unknown.
In a stochastic setting the reward of an arm is sampled from some unknown
distribution, and in an adversarial setting the reward of an arm is chosen by
an adversary and is not necessarily sampled from any distribution.
Particularly, in this survey we are interested in the situation where we observe side information at
each time $t$. We call this side information the context.
The arm that has the highest expected reward may be different given different contexts. This
variant of multi-armed bandits is called contextual bandits. 

Usually in a contextual bandits problem there is a set of policies, and each
policy maps a context to an arm. There can be infinite number of policies,
especially when reducing bandits to classification problems. We define the
regret of the agent as the gap between the highest expected cumulative reward
any policy can achieve and the cumulative reward the agent actually get. The
goal of the agent is to minimize the regret.
Contextual bandits can naturally model many problems. For example, in a news
personalization system, we can treat each news articles as an arm, and the
features of both articles and users as contexts. The agent then picks
articles for each user to maximize click-through rate or dwell time.

There are a lot of bandits algorithms, and it is always important to know what
they are competing with. For example, in K-armed bandits, the agents are
competing with the arm that has the highest expected reward; and in 
contextual bandits with expert advice, the agents are competing with the expert that has
the highest expected reward; and when we reduce contextual bandits to 
classification/regression problems, the agents are competing with the best
policy in a pre-defined policy set.

As a overview, we summarize all the algorithms we will talk about in Table
\ref{table:algocompare}. 
In this table, $C$ is the number of distinct contexts, $N$ is the number of policies, $K$ is the number of arms, and $d$ is the dimension of contexts. Note
that the second last column shows if the algorithm requires the knowledge of
$T$, and it
doesn't necessary mean that the algorithm requires the knowledge of $T$ to run,
but means that to achieve the proposed regret the knowledge of $T$ is
required.
\newpage
\begin{landscape}
\begin{table}[]
\centering
\caption{A comparison between all the contextual bandits algorithm we will talk about}
\label{table:algocompare}
\renewcommand{\arraystretch}{2}
\begin{tabular}{|l|l|l|l|l|l|l|}
\hline
Algorithm & Regret & \pbox{20cm}{With Hight\\Probability} & \pbox{20cm}{Can Have\\Infinite Policies} & \pbox{20cm}{Need to \\ know T} & \pbox{20cm}{adversarial\\ reward} \\ \hline
Reduce to MAB                              & $O\left(\sqrt{TCK\ln K}\right)$ or
                                             $O\left(\sqrt{TN \ln N}\right)$ & no &   no     & yes & yes             \\ \hline
EXP4                                       & $O\left(\sqrt{TK\ln N}\right)$  & no &   no       & yes & yes           \\ \hline
EXP4.P                                     & $O\left(\sqrt{TK\ln (N/\delta)}\right)$  & yes&   no  & yes & yes\\ \hline
LinUCB                                     & $O\left(d\sqrt{T\ln((1+T)/\delta)}\right)$ & yes  &            yes  & yes & no      \\ \hline
SupLinUCB                                     & $O\left(\sqrt{Td\ln^3(KT\ln T/\delta)}\right)$ & yes  &            yes  & yes & no      \\ \hline
SupLinREL                                     & $O\left( \sqrt{Td}(1+\ln (2KT\ln T / \delta))^{3/2} \right)$ & yes  &            yes  & yes & no      \\ \hline
GP-UCB                                     & $\tilde{O}\left(\sqrt{T}\left(
                                             B\sqrt{\gamma_T}+\gamma_T \right) \right)$ & yes  &            yes  & yes & no      \\ \hline
KernelUCB                                     & $\tilde{O}(\sqrt{B\tilde{d}T})$ & yes  &            yes  & yes & no      \\ \hline
Epoch-Greedy                               & $O\left((K\ln(N/\delta))^{1/3}T^{2/3}\right)$ & yes& yes    & no & no \\ \hline
Randomized UCB                             & $O\left(\sqrt{TK\ln(N/\delta)}\right)$ & yes & yes   &no & no  \\ \hline
ILOVETOCONBANDITS                             &  $O\left(\sqrt{TK\ln(N/\delta)}\right)$ & yes  & yes   & no & no     \\ \hline
\pbox{10cm}{Thompson Sampling\\with Linear Regression}   & $O\left(\frac{d^2}{\epsilon}\sqrt{T^{1+\epsilon}} (\ln(Td)\ln\frac{1}{\delta})\right)$ & yes &  yes & no & no                    \\ \hline
\end{tabular}
\end{table}
\end{landscape}

\section{Unbiased Reward Estimator}
\label{sec:unbiased_estimator_trick}
One challenge of bandits problems is that we only observe partial
feedback. 
Suppose at time $t$ the algorithm randomly selects an arm $a_t$ based on a probability
vector $\ve{p_t}$.
Denote the true reward vector by $\ve{r_t} \in [0, 1]^K$ and
the reward vector we observed by $\ve{r'_t} \in [0, 1]^K$,
then all the elements in $\ve{r'_t}$ are zero except $r'_{t, a_t}$ which is equal
to $r_{t, a_t}$. Then $\ve{r'_t}$ is certainly not a
unbiased estimator of $\ve{r_t}$ because $\mathrm{E}(r'_{t, a_t}) = p_{a_t} \cdot
r_{t, a_t} \ne r_{t, a_t}$. A common trick to this is to use $\hat{r}_{t, a_t} = r'_{t, a_t}/p_{a_t}$
instead of $r'_{a_t}$. In this way we get a unbiased estimator of the true reward
vector $\ve{r_t}$: for any arm $a$ 
\begin{align*}
\mathrm{E} (\hat{r}_{t, a}) &= p_{a} \cdot r_{t, a}/p_{a} + (1-p_a) * 0 \\
&=r_{t, a}
\end{align*}
The expectation is with respect to the random choice of arms at time
$t$. This trick is used by many algorithms described later.

\section{Reduce to K-Armed Bandits}
If it is possible to enumerate all the contexts, then one naive
way is to apply a K-armed bandits algorithm to each
context. However, in this way we ignore all the relationships between contexts since
we treat them independently.

Suppose there are $C$ distinct contexts in the context set $\mathcal{X}$, and the context
at time $t$ is $x_t \in \{1, 2, ..., C\}$. Also assume there are $K$ arms in the
arm set $\mathcal{A}$ and the arm selected at time $t$ is $a_t \in \{1, 2, ..., K\}$. Define the
policy set to be all the possible mappings from contexts to arms as $\Pi = \{f:
\mathcal{X} \rightarrow \mathcal{A}\}$, then the regret of the agent is defined
as: 
\begin{align}
R_T = \sup_{f\in \Pi} \mathrm{E}\left[ \sum_{t=1}^T(r_{t, f(x_t)} - r_{t,
  a_t}) \right] \label{eq:sec3regret}
\end{align}
\begin{theorem}
Apply EXP3 \citep{auer2002nonstochastic}, a non-contextual multi-armed bandits
algorithm, on each context, then the regret is
\begin{align*}
R_T \le 2.63\sqrt{TCK\ln K}
\end{align*}
\end{theorem}
\begin{proof}
Define $n_i = \sum_{t=1}^{T}\mathds{I}(x_t = i)$, then $\sum_{i=1}^C n_i = T$.
We know that the regret bound of EXP3 algorithm is $2.63 \sqrt{TK\ln K}$, so 
\begin{align*}
R_T &= \sup_{f\in \Pi} \mathrm{E}\left[ \sum_{t=1}^T(r_{t, f(x_t)} -
  r_{t, a_t}) \right] \\
  &= \sum_{i=1}^C \sup_{f\in \Pi} \mathrm{E}\left[ \sum_{t=1}^{T}\mathds{I}(x_t=i)(r_{t,
    f(x_t)} - r_{t, a_t}) \right] \\
&\le \sum_{i=1}^C 2.63\sqrt{n_i K\ln K} \\
&\le 2.63\sqrt{TCK \ln K} \text{\quad (Cauchy-Schwarz inequality)}
\end{align*}
\end{proof}
One problem with this method is that it assumes the contexts can be enumerated,
which is not true when contexts are continuous. Also this algorithm treats
each context independently, so learning one of them does not help learning the
other ones.

If there exists is a set of pre-defined policies and we want to compete with
the best one, then another way to
reduce to K-armed bandits is to treat each policy as an arm and then apply EXP3
algorithm. The regret is still defined as Equation (\ref{eq:sec3regret}), but
$\Pi$ is now a pre-defined policy set instead of all possible mappings from contexts
to arms. Let $N$ be the number of polices in the policy set, then by applying
EXP3 algorithm we get the regret bound $O(\sqrt{TN\ln N})$. This algorithm works
if we have small number of policies and large number of arms; however, if we
have a huge number of policies, then this regret bound is weak.

\section{Stochastic Contextual Bandits}
Stochastic contextual bandits algorithms assume that the reward of each arm
follows an unknown probability distribution. Some algorithms further assume such
distribution is sub-Gaussian with unknown parameters. In this section, we
first talk about stochastic contextual bandits algorithms with linear realizability
assumption; In this case, the expectation of the reward of each arm is linear
with respect to the arm's features. Then we talk about algorithms
that work for arbitrary set of policies without such assumption.

\subsection{Stochastic Contextual Bandits with Linear Realizability Assumption}
\label{sec:scb_linear}
\subsubsection{LinUCB/SupLinUCB}
\label{sec:scb_linucb}
LinUCB \citep{li2010contextual, chu2011contextual} extends UCB algorithm to
contextual cases. Suppose each arm is associated with a feature vector
$x_{t, a} \in \mathrm{R}^d$. In news recommendation, $x_{t,a}$ could be
user-article pairwise feature vectors. LinUCB assumes the expected reward of an arm $a$
is linear with respect to its feature vector $x_{t, a} \in \mathrm{R}^d$:
\begin{align*}
\mathrm{E}[r_{t, a}|x_{t, a}] = x_{t,a}^\top \theta^*
\end{align*}
where $\theta^*$ is the true coefficient vector. The noise $\epsilon_{t, a}$ is assumed to
be R-sub-Gaussian for any $t$. Without loss of generality, we
assume $||\theta^*|| \le S$ and $||x_{t,a}|| \le L$, where $||\cdot||$ denotes
the $\ell_2$-norm. We also assume the reward $r_{t, a} \le 1$. Denote the best
arm at time $t$ by $a_t^* = \argmax_a x_{t,a}^\top \theta^*$, and the arm
selected by the algorithm at time $t$ by $a_t$, then the T-trial regret of
LinUCB is defined as
\begin{align*}
  R_T &= \mathrm{E}\left[\sum_{t=1}^T r_{t, a_t^*} - \sum_{t=1}^T r_{t, a_t} \right] \\
&= \sum_{t=1}^T x_{t, a_t^*}^\top \theta^* - \sum_{t=1}^T x_{t, a_t}^\top \theta^*
\end{align*}

Let $D_t \in \mathrm{R}^{t \times d}$ and $c_t \in \mathrm{R}^t$ be the
historical data up to time $t$, where the $i^{th}$ row of $D_t$ represents the
feature vector of the arm pulled at time $i$, and the $i^{th}$ row of $c_t$
represents the corresponding reward.
If samples $(x_{t, a}, r_{t, {a_t}})$ are independent, then we can get a closed-form estimator of
$\theta^*$ by ridge regression:
\begin{align*}
\hat{\theta}_t = (D_t^\top D_t + \lambda \mathrm{I}_d)^{-1} D_t^\top c_t
\end{align*}
The accuracy of the estimator, of course, depends on the amount of data. 
\citet{chu2011contextual} derived a upper confidence bound for the prediction
$x_{t,a}^\top \hat{\theta}_t$:
\begin{theorem}
  \label{thm:linucbucb}
  Suppose the rewards $r_{t, a}$ are independent random variables with means
  $\mathrm{E}[r_{t, a}] = x_{t, a}^\top \theta^*$, let $\epsilon=\sqrt{\frac{1}{2}\ln\frac{2TK}{\delta}}$ and $A_t =
  D_t^\top D_t + \mathrm{I}_d$ then with
  probability $1-\delta/T$, we have
\begin{align*}
|x_{t, a}^\top \hat{\theta}_t - x_{t, a}^\top \theta^*| \le (\epsilon+1)\sqrt{x_{t,
  a}^\top A_t^{-1}x_{t, a}}
\end{align*} 
\end{theorem}
LinUCB always selects the arm with the
highest upper confidence bound. The algorithm is described in Algorithm
\ref{algo:linucb}.
\begin{algorithm}[!htp]
\caption{LinUCB}
\label{algo:linucb}
\begin{algorithmic}
  \Require $\alpha > 0, \lambda > 0$
  \State $A = \lambda \mathrm{I}_d$
  \State $b = \ve{0}_d$
  \For{t=1, 2, ..., T}
  \State $\theta_t = A^{-1}b$
  \State Observe features of all $K$ arms $a\in \mathcal{A}_t : x_{t, a} \in
  \mathrm{R}^d$
  \For{a=1, 2, ... K}
  \State $s_{t, a} = x_{t, a}^\top \theta_t + \alpha \sqrt{x_{t, a}^\top A^{-1}_t x_{t,
    a}}$
  \EndFor
  \State Choose arm $a_t = \argmax_{a} s_{t, a}$, break ties arbitrarily
  \State Receive reward $r_t \in [0, 1]$
  \State $A = A + x_{t, a} x_{t, a}^\top$
  \State $b = b + x_{t, a} r_t$
  \EndFor
\end{algorithmic}
\end{algorithm}

However, LinUCB algorithm use samples from previous rounds to estimate $\theta^*$
and then pick a sample for current round. So the samples are not independent.
In \cite{abbasi2011improved} it was shown through martingale techniques that
concentration results for the predictors can be obtained directly without
requiring the assumption that they are built as linear combinations of
independent random variables. 
\begin{theorem}[\citet{abbasi2011improved}]\label{HighProbThm}
Let the noise term $\epsilon_{t, a}$ be R-sub-Gaussian where $R \ge 0$ is a
fixed constant. With probability at least $1-\delta$, $\forall t\geq 1,$
\begin{align*}
\|\hat{\theta}_t-\theta^*\|_{A_t}
		\leq&  R \sqrt{2\log \left(\frac{|A_t|^{1/2}}
  {\lambda^{1/2}\delta}\right)} +\lambda^{1/2}S.
\end{align*}
\end{theorem}
We can now choose appropriate values of $\alpha_t$ for LinUCB as the right side
of the inequality in Theorem \ref{HighProbThm}. Note that here $\alpha$
depends on $t$, we denote so it is a little different than the original LinUCB algorithm (Algorithm
\ref{algo:linucb}) which has independent assumption.

\begin{theorem}\label{RegretThm}
Let $\lambda \ge max(1, \mathrm{L}^2)$. The cumulative regret of LinUCB is with probability at least $1-\delta$ bounded as:
\begin{align*}
R_T &\leq \sqrt{T  d \log (1 + TL^2/(d \lambda))} \times \\
 &\times \left( R \sqrt{d\log(1 + TL^2/ (\lambda d)) + 2\log(1/\delta)}  +\lambda^{1/2}S \right) 
\end{align*}
\end{theorem}
To proof Theorem \ref{RegretThm}, We first state two technical lemmas from
\cite{abbasi2011improved}:
\begin{lemma}[\citet{abbasi2011improved}]\label{lemma:sumx2}
We have the following bound:
\[ \sum_{t=1}^{T}\|x_t\|_{A_t^{-1}}^2\leq 2 \log \frac{|A_t|}{\lambda}. \] 
\end{lemma}
\begin{lemma}[\citet{abbasi2011improved}]\label{lemma:det}
The determinant $|A_t|$ can be bounded as:
 \[ |A_t| \leq (\lambda + tL^2/d) ^ d.  \] 
\end{lemma}

We can now simplify $\alpha_t$ as 
\begin{align*}
\alpha_t &\leq  R \sqrt{2\log \left(|A_t|^{1/2} \lambda^{-1/2}\delta^{-1}\right)} +\lambda^{1/2} S \\
  &\leq R \sqrt{d\log(1 + TL^2/ (\lambda d)) + 2\log(1/\delta)}  +\lambda^{1/2} S
\end{align*}
where $d\ge 1$ and $\lambda \ge \max(1, L^2)$ to have  $\lambda^{1/d} \ge \lambda$.

\begin{proof} {[Theorem~\ref{RegretThm}]}
Let $\bar r_t$ denote the instantaneous regret at time $t$. With probability at least $1-\delta$, for all $t$:
\begin{align}
\bar r_t  &= x_{t,*}^\top\theta^* - x_t^\top\theta^*  \nonumber \\
	&\leq  x_t^\top\hat\theta_t + \alpha_t\|x_t\|_{A_t^{-1}} - x_t^T\theta^* \label{ThRuseOFU} \\
        &\leq  x_t^\top\hat\theta_t + \alpha_t\|x_t\|_{A_t^{-1}} - x_t^\top\hat\theta_t + \alpha_t\|x_t\|_{A_t^{-1}}   \label{ThRuseConf} \\
	& = 2  \alpha_t\|x_t\|_{A_t^{-1}} \nonumber
\end{align}

The inequality \eqref{ThRuseOFU} is by the algorithm design and
reflects the optimistic principle of LinUCB. Specifically,
$ x_{*}^\top\hat{\theta}_t + \alpha_t\|x_{*}\|_{A_t^{-1}} \leq x_t^\top\hat\theta_t + \alpha_t\|x_{t}\|_{A_t^{-1}},
$
from which:
$$
 x_{*}^\top\theta^*  \leq x_{*}^\top\hat{\theta}_t + \alpha_t\|x_{*}\|_{A_t^{-1}}\leq x_t^\top\hat\theta_t + \alpha_t\|x_{t}\|_{A_t^{-1}}  \\
$$

In \eqref{ThRuseConf}, we applied Theorem~\ref{HighProbThm} to get:
\[
 x_t^\top\hat\theta_t \leq x_{t,*}^\top\theta^*    + \alpha_t\|x_t\|_{A_t^{-1}}
\]

Finally by Lemmas~\ref{lemma:sumx2} and~\ref{lemma:det}:
\begin{align*}
R_T &= \sum_{t=1}^T \bar r_t \leq \sqrt{T \sum_{t=1}^T \bar r_t^2}  \\
   &\leq 2 \alpha_T \sqrt{T  \sum_{t=1}^T \|x_t\|^2_{A_t^{-1}}} \\
   &\leq 2 \alpha_T \sqrt{T  \log\frac{|A_t|}{\lambda}} \\
    &\leq 2 \alpha_T \sqrt{T ( d \log (\lambda + TL^2 /d) - \log \lambda) }  \\
    &\leq 2 \alpha_T \sqrt{T  d \log (1 + TL^2 /(d \lambda))}
\end{align*}

Above we used that $\alpha_t \leq \alpha_T$ because $\alpha_t$ is not decreasing $t$.
Next we used that $\lambda \ge \max(1, L^2)$ to have  $\lambda^{1/d} \ge \lambda$.
By plugging $\alpha_t$, we get:
\begin{align*}
R_T &\leq \sqrt{T  d \log (1 + T L^2/(d \lambda))} \times \\
 &\times \left( R \sqrt{d\log(1 + TL^2/ (\lambda d)) + 2\log(1/\delta)}  +\lambda^{1/2} S \right) \\
 &= O( d \sqrt{T \log ((1 + T)/\delta)})
\end{align*}
\end{proof}

Inspired by \citet{auer2003using}, \citet{chu2011contextual} proposed SupLinUCB
algorithm, which is a variant of LinUCB. It is mainly used for theoretical
analysis, but not a practical algorithm. SupLinUCB constructs S sets to store
previously pulled arms and rewards. The algorithm are designed so that within the same set
the sequence of feature vectors are fixed and the rewards are independent. As a
results, an arm's predicted reward in the current round is a linear combination of
rewards that are independent random variables, and so Azuma's inequality can be
used to get the regret bound of the algorithm.

\citet{chu2011contextual} proved that with probability at least $1-\delta$, the
regret bound of SupLinUCB is $O\left( \sqrt{Td\ln^3(KT\ln (T)/\delta)}  \right)$.

\subsubsection{LinREL/SupLinREL}
The problem setting of LinREL \citep{auer2003using} is the same as LinUCB, so we use the same
notations here. LinREL and LinUCB both assume that for each
arm there is an associated feature vector $x_{t, a}$ and the expected reward of
arm $a$ is linear with respect to its feature vector: $\mathrm{E}[r_{t, a}|x_{t,
a}] = x_{t, a}^\top \theta^*$, where $\theta^*$ is the true coefficient
vector. However, these two algorithms take two different forms of
regularization. LinUCB takes a $\ell_2$ regularization term similar to ridge
regression; that is, it adds a diagonal matrix $\lambda \mathrm{I}_d$ to matrix
$D_t^\top D_t$. LinREL, on the other hand, do regularization by setting
$D_t^\top D_t$ matrix's small eigenvalues to zero. LinREL algorithm is described
in Algorithm \ref{algo:linrel}. We have the following theorem to show that
Equation (\ref{eq:linrel_sta}) is the upper confidence bound of the true reward
of arm $a$ at time $t$. Note that the following theorem assumes the rewards
observed at each time $t$ are independent random variables. However, similar to
LinUCB, this assumption is not true. We will deal with this problem later.

\begin{algorithm}[!htp]
\caption{LinREL}
\label{algo:linrel}
\begin{algorithmic}
  \Require $\delta \in [0, 1]$, number of trials $T$.
  \State Let $D_t \in \mathrm{R}^{t \times d} \text{ and } c_t \in \mathrm{R}^t$ be the
  matrix and vector to store previously pulled arm feature vectors and rewards.
  \For{t=1, 2, ..., T}
  \State Calculate eigendecomposition
  \begin{align*}
    D_t^\top D_t  = U_t^\top \text{diag}(\lambda_t^1, \lambda_t^2, ...,
    \lambda_t^d) U_t
  \end{align*}
  \State where $\lambda_t^1, ..., \lambda_t^k \ge 1$, $\lambda_t^{k+1}, ...,
  \lambda_t^d < 1$, and $U_t^\top \cdot U_t = \text{I}_d$
  \State Observe features of all $K$ arms $a\in \mathcal{A}_t : x_{t, a} \in
  \mathrm{R}^d$
  \For{a=1, 2, ... K}
  \State
  \begin{align}
    \tilde{x}_{t, a} &= (\tilde{x}_{t, a}^1, ..., \tilde{x}_{t, a}^d) = U_t x_{t, a} \nonumber \\
    \tilde{u}_{t, a} &= (\tilde{x}_{t, a}^1, ..., \tilde{x}_{t, a}^k, 0, ..., 0)^\top \nonumber \\
    \tilde{v}_{t, a} &= (0, ..., 0, \tilde{x}_{t, a}^{k+1}, ..., \tilde{x}_{t, a}^d)^\top \nonumber \\
    w_{t, a} &= \left(\tilde{u}_{t, a}^\top \cdot
               \text{diag}\left(\frac{1}{\lambda_t^1}, ...,
               \frac{1}{\lambda_t^k}, 0, ..., 0\right) \cdot U_t \cdot D_t^\top \right)^\top \label{eq:linrel_wta} \\
    s_{t, a} &= w_{t, a}^\top c_t + ||w_{t,a}||\left(\sqrt{\ln(2TK/\delta)}\right) + ||\tilde{v}_{t, a}|| \label{eq:linrel_sta}
  \end{align}
  \EndFor
  \State Choose arm $a_t = \argmax_{a} s_{t, a}$, break ties arbitrarily
  \State Receive reward $r_t \in [0, 1]$, append $x_{t, a}$ and $r_{t, a}$ to
  $D_t$ and $c_t$.
  \EndFor
\end{algorithmic}
\end{algorithm}

\begin{theorem}
\label{thm:linrel_ucb}
Suppose the rewards $r_{\tau, a}, \tau \in {1,...,t-1}$ are independent random variables
with mean $\mathrm{E}[x_{\tau, a}] = x_{\tau, a}^\top \theta^*$. Then at time t,
with probability $1-\delta/T$ all arm $a \in \mathcal{A}_t$ satisfy
\begin{align*}
|w_{t,a}^\top c_t - x_{t,a}^\top \theta^*| \le ||w_{t, a}|| \left(
  \sqrt{2\ln(2TK/\delta)}  \right) + ||\tilde{v}_{t, a}||
\end{align*}
\end{theorem}
  Suppose $D_t^\top D_t$ is invertible, then we can estimate the model parameter
  $\hat{\theta} = (D_t^\top D_t)^{-1}D^\top c_t$. Given a feature vector $x_{t,
    a}$, the predicted reward is 
\begin{align*}
r_{t, a} = x_{t_a}^\top \hat{\theta} = (x_{t,a}^\top (D_t^\top D_t)^{-1}D^\top) c_t
\end{align*}
So we can view $r_{t,a}$ as a linear combination of previous rewards. In
Equation (\ref{eq:linrel_wta}), $w_{t, a}$ is essentially the weights for each
previous reward (after regularization). We use $w_{t, a}^\tau$ to denote the weight
of reward $r_{\tau, a}$.
\begin{proof}[Theorem \ref{thm:linrel_ucb}]
Let $z_{\tau} = r_{t, a} \cdot w_{t,a}^\tau$, then $|z_{\tau}| \le
w_{t,a}^\tau$, and
\begin{gather*}
w_{t,a}^\top c_t = \sum_{\tau=1}^{t-1} z_{\tau} = \sum_{\tau=1}^{t-1}r_{t,a}
  \cdot w_{t, a}^\tau \\
\sum_{\tau=1}^{t-1}\mathrm{E}[z_{\tau}|z_1, ..., z_{\tau-1}] =
  \sum_{\tau=1}^{t-1}\mathrm{E}[z_{\tau}] = \sum_{\tau=1}^{t-1} x_{\tau, a}^\top
  \theta^* \cdot w_{t,a}^\tau
\end{gather*}
Apply Azuma's inequality we have
\begin{align*}
  &\mathrm{P}\left( w_{t,a}^\top c_t - \sum_{\tau=1}^{t-1} x_{\tau,
  a}^\top\theta^* \cdot w_{t,a}^\tau \ge ||w_{t,a}||
  \left(\sqrt{2\ln(2TK/\delta)}\right)\right) \\
  &=\mathrm{P}\left( w_{t,a}^\top c_t - w_{t,a}^\top D_t \theta^* \ge ||w_{t,a}||
  \left(\sqrt{2\ln(2TK/\delta)}\right)\right) \\
&\le \frac{\delta}{TK}
\end{align*}
Now what we really need is the inequality between $w_{t,a}^\top c_t$ and $x_{t,
  a}^\top \theta^*$. Note that 
\begin{align*}
x_{t,a} &= U_t^\top \tilde{x}_{t,a} \\
&= U_t^\top \tilde{u}_{t,a} + U_t^\top \tilde{v}_{t,a} \\
&=D_t^\top D_t (D^\top D_t)^{-1} D^\top_t w_{t,a} + U_t^\top \tilde{v}_{t,a} \\
&=D^\top_t w_{t, a} + U_t^\top \tilde{v}_{t,a}
\end{align*}
Assuming $||\theta^*|| \le 1$,  we have
\begin{align*}
  \mathrm{P}\left( w_{t,a}^\top c_t - x_{t,a}^\top \theta^* \ge ||w_{t,a}||
  \left(\sqrt{2\ln(2TK/\delta)}\right) + ||\tilde{v}_{t,a}||\right)
\le \frac{\delta}{TK}
\end{align*}
Take the union bound over all arms, we prove the theorem.
\end{proof}

The above proof uses the assumption that all the rewards observed are
independent random variables. However in LinREL, the actions taken in previous
rounds will influence the estimated $\hat{\theta}$, and thus influence the
decision in current round. To deal with this problem, \citet{auer2003using}
proposed SupLinREL algorithm. SupLinREL construct S sets $\Psi_t^1, ...,
\Psi_t^S$, each set $\Psi_t^s$ contains arm
pulled at stage $s$. It is designed so that the rewards of arms inside one stage
is independent, and within one stage they apply LinREL algorithm. They proved
that the regret bound of SupLinREL is $O\left( \sqrt{Td}(1+\ln (2KT\ln T))^{3/2} \right)$.
\subsubsection{CofineUCB}
\subsubsection{Thompson Sampling with Linear Payoffs}
Thompson sampling is a heuristic to balance exploration and exploitation, and it
achieves good empirical results on display ads and news recommendation
\citep{chapelle2011empirical}. Thompson sampling can be applied to both
contextual and non-contextual multi-armed bandits problems. For example
\citet{agrawal2013further} provides a $O(\sqrt{NT\ln T})$ regret bound for
non-contextual case. Here we focus on the contextual case.

Let $\mathcal{D}$ be the set of past observations $(x_t, a_t, r_t)$, where $x_t$
is the context, $a_t$ is the arm pulled, and $r_t$ is the reward of that
arm. Thompson sampling assumes a parametric likelihood function $P(r|a, x,
\theta)$ for the reward, where $\theta$ is the model parameter. We denote the
true parameters by $\theta^*$. Ideally, we would choose an arm that maximize
the expected reward $\max_a \mathrm{E}(r|a, x, \theta^*)$, but of course we
don't know the true parameters. Instead Thompson sampling apply a prior believe
$P(\theta)$ on parameter $\theta$, and then based on the data observed, it
update the posterior
distribution of $\theta$ by $P(\theta|\mathcal{D}) \propto P(\theta)
\prod_{t=1}^T P(r_t|x_t, a_t, \theta)$. Now if we just want to maximize the
immediate reward, then we would choose an arm that maximize $\mathrm{E}(r |
a, x) = \int \mathrm{E}(a, x, \theta)P(\theta|\mathcal{D})d\theta$, but in an
exploration/exploitation setting, we want to choose an arm according to its
probability of being optimal. So Thompson sampling randomly selects an action $a$ according to 
\begin{align*}
\int \mathbb{I} \left[ E(r|a, \theta) = \max_{a'} E(r|a', \theta) \right]
  P(\theta|\mathcal{D}) d\theta
\end{align*}
In the actual algorithm, we don't need to calculate the integral, it suffices to draw a random parameter $\theta$ from posterior
distribution and then select the arm with highest reward under that
$\theta$. The general framework of Thompson sampling is described in Algorithm \ref{algo:ts}.
\begin{algorithm}[!htp]
\caption{General Framework of Thompson Sampling}
\label{algo:ts}
\begin{algorithmic}
  \State Define $\mathcal{D} = \{\}$
  \For{$t=1, ..., T$}
  \State Receive context $x_t$
  \State Draw $\theta_t$ from posterior distribution $P(\theta|\mathcal{D})$
  \State Select arm $a_t=\argmax_a\mathrm{E}(r| x_t, a, \theta_t)$
  \State Receive reward $r_t$
  \State $\mathcal{D} = \mathcal{D} \cup \{x_t, a_t, r_t\}$
  \EndFor
\end{algorithmic}
\end{algorithm}

\noindent According to the prior we choose or the likelihood function we use, we can have
different variants of Thompson sampling. In the following section we introduce two of them.

\citet{agrawal2012thompson} proposed a Thompson sampling algorithm with linear
payoffs. Suppose there are a total of $K$ arms, each arm $a$ is associated with a d-dimensional feature vector
$\ve{x}_{t, a}$ at time $t$. Note that $\ve{x}_{t, a} \ne \ve{x}_{t', a}$. There
is no assumption on the distribution of $\ve{x}$, so the context can be chosen by
an adversary. A
linear predictor is defined by a d-dimensional parameter $\ve{\mu} \in
\mathrm{R}^d$, and predicts the mean reward of arm $a$ by $\ve{\mu}
\cdot \ve{x}_{t, a}$. \citet{agrawal2012thompson} assumes an unknown underlying
parameter $\ve{\mu}^* \in \mathrm{R}^d$ such that the expected reward for arm
$a$ at time $t$ is
$\bar{r}_{t, a} = \ve{\mu^*} \cdot \ve{x}_{t,a}$. The real reward $r_{t, a}$ of
arm $a$ at time $t$ is generated from an
unknown distribution with mean $\bar{r}_{t, a}$. At each time $t \in \{1, ..., T\}$ the
algorithm chooses an arm $a_t$ and receives reward $r_t$. Let $a^*$ be the optimal arm at time
$t$:
\begin{align*}
a^*_t = \argmax_a \bar{r}_{t, a}
\end{align*}
and $\Delta_{t,
a}$ be the difference of the expected reward between the optimal arm and arm
$a$:
\begin{align*}
  \Delta_{t, a} = \bar{r}_{t, a^*} - \bar{r}_{t, a}
\end{align*}
Then the regret of the algorithm is defined as:
\begin{align*}
R_T = \sum_{t=1}^T \Delta_{t, a_t}
\end{align*}
In the paper they assume $\delta_{t, a} = r_{t, a} - \bar{r}_{t, a}$
is conditionally R-sub-Gaussian, which means for a constant $R \ge 0$, $r_{t, a}
\in [\bar{r}_{t, a} - R, \bar{r}_{r, t} + R]$.
There are many likelihood distributions that satisfy this R-sub-Gaussian
condition. But to make the algorithm simple, they use Gaussian likelihood and Gaussian prior. The
likelihood of reward $\bar{r}_{t, a}$ given the context $\ve{x}_{t,a}$ is given by the
pdf of Gaussian distribution $\mathcal{N}(\ve{x}_{t, a}^\top \ve{\mu}^*, v^2)$. $v$ is
defined as $v=R\sqrt{\frac{24}{\epsilon}d\ln(\frac{t}{\delta})}$, where $\epsilon
\in (0, 1)$ is the algorithm parameter and $\delta$ controls the high probability
regret bound. Similar to the closed-form of linear regression, we define 
\begin{align*}
B_t &= I_d + \sum_{\tau=1}^{t-1} \ve{x}_{\tau, a}\ve{x}_{\tau, a}^\top \\
\hat{\ve{\mu}}_{t} &= B_t^{-1} \left(\sum_{\tau=1}^{t-1} \ve{x}_{\tau, a} r_{\tau, a}\right)
\end{align*}
Then we have the following theorem:
\begin{theorem}
\label{thm:tslppp}
if the prior of $\ve{\mu}^*$ at time $t$ is defined as
$\mathcal{N}(\hat{\ve{\mu}}_t, v^2B_t^{-1})$, then the posterior of $\ve{\mu}^*$ is
$\mathcal{N}(\hat{\ve{\mu}}_{t+1}, v^2 B_{t+1}^{-1})$.
\end{theorem}
\begin{proof}
\begin{align*}
P(\ve{\mu} | r_{t, a}) &\propto P(r_{t,a}|\ve{\mu})P(\ve{\mu}) \\
&\propto \exp\left( -\frac{1}{2v^2}((r_{t,a}-\ve{\mu}^\top x_{t, a} + (\ve{\mu} -
  \ve{\hat{\mu}}_t)^\top B_t (\ve{\mu} - \ve{\hat{\mu}}_t) \right) \\
&\propto \exp\left( -\frac{1}{2v^2}(\ve{\mu}^\top B_{t+1} \ve{\mu} -
  2\ve{\mu}^\top B_{t+1}\ve{\hat{\mu}}_{t+1} ) \right) \\
&\propto \exp\left( -\frac{1}{2v^2}(\ve{u} - \ve{\hat{\mu}}_{t+1})^\top B_{t+1}
  (\ve{u} - \ve{\hat{\mu}}_{t+1}) \right) \\
&\propto \mathcal{N}(\ve{\hat{\mu}}_{t+1}, v^2B_{t+1}^{-1})
\end{align*}
\end{proof}

\noindent Theorem \ref{thm:tslppp} gives us a way to update our believe about
the parameter after observing new data. The algorithm is described in Algorithm \ref{algo:tslp}.

\begin{algorithm}[]
\caption{Thompson Sampling with Linear Payoff}
\label{algo:tslp}
\begin{algorithmic}
\Require $\delta \in (0, 1]$
\State Define $v = R\sqrt{\frac{24}{\epsilon}d\ln(\frac{t}{\delta})}, B = I_d, \ve{\hat{\mu}} = 0_d, f = 0_d$
\For{$t=1, 2..., T$}
\State Sample $\ve{u}_t$ from distribution $\mathcal{N}(\ve{\hat{\mu}}, v^2B^{-1})$
\State Pull arm $a_t = \argmax_a \ve{x}_{t, a}^\top \ve{u}_t$ 
\State Receive reward $r_t$
\State Update:
\begin{gather*}
B = B + \ve{x}_{t, a}\ve{x}_{t, a}^\top \\
\ve{f} = \ve{f} + \ve{x}_{t, a} r_t \\
\ve{\hat{\mu}} = B^{-1} \ve{f}
\end{gather*}
\EndFor
\end{algorithmic}
\end{algorithm}

\begin{theorem}
  With probability $1-\delta$, the regret is bounded by:
\begin{align*}
R_T = O\left(\frac{d^2}{\epsilon}\sqrt{T^{1+\epsilon}} (\ln(Td)\ln\frac{1}{\delta})\right)
\end{align*}
\end{theorem}

\citet{chapelle2011empirical} described a way of doing Thompson sampling with
logistic regression. Let $\ve{w}$ be the weight vector of logistic regression
and $w_i$ be the $i^{th}$ element. Each $w_i$ follows a Gaussian distribution
$w_i \sim \mathcal{N}(m_i, q_i^{-1})$. They apply Laplace approximation to get
the posterior distribution of the weight vector, which is a Gaussian
distribution with diagonal covariance matrix. The algorithm is described in
Algorithm \ref{algo:tslr}.
\begin{algorithm}[]
\caption{Thompson Sampling with Logistic Regression}
\label{algo:tslr}
\begin{algorithmic}
  \State Require $\lambda \ge 0$, batch size $S \ge 0$
  \State Define $\mathcal{D} = \{\}$, $m_i = 0, q_i = \lambda$ for all elements
  in the weight vector $\ve{w} \in \mathrm{R}^d$.
  \For{each batch $b=1, ..., B$} \quad $\triangleright$ Process in mini-batch style
  \State Draw $\ve{w}$ from posterior distribution $\mathcal{N}(\ve{m}, \text{diag}(\ve{q})^{-1})$
  \For{$t=1, ..., S$}
  \State Receive context $\ve{x}_{b, t, j}$ for each article $j$.
  \State Select arm $a_t=\argmax_j 1/(1+\exp(-\ve{x}_{b,t,j}\cdot \ve{w}))$
  \State Receive reward $r_t \in \{0, 1\}$
  \State $\mathcal{D} = \mathcal{D} \cup \{x_{b, t, a_t}, a_t, r_t\}$
  \EndFor
  \State Solve the following optimization problem to get $\ve{\bar{w}}$
  \begin{align*}
    \frac{1}{2} \sum_{i=1}^d q_i(\bar{w}_i-m_i)^2 + \sum_{(\ve{x}, r) \in \mathcal{D}} \ln
    (1+\exp(-r \ve{\bar{w}} \cdot \ve{x}))
    \end{align*}
  \State Set prior for next block
  \State \begin{gather*}
  m_i = \bar{w}_i\\
  q_i = q_i + \sum_{(\ve{x}, r) \in \mathcal{D}} x_{i}^2 p_j(1-p_j), p_j =
  (1+\exp(-\ve{\bar{w}} \cdot
  \ve{x}))^{-1}\end{gather*}
  \EndFor
\end{algorithmic}
\end{algorithm}
\citet{chapelle2011empirical} didn't give a regret bound for this algorithm, but
showed that it achieve good empirical results on display advertising.

\subsubsection{SpectralUCB}

\subsection{Kernelized Stochastic Contextual Bandits}
Recall that in section \ref{sec:scb_linear} we assume a linear relationship
between the arm's features and the expected reward: $\mathrm{E}(r) =
x^\top \theta^*$; however, linearity assumption is not always
true. Instead, in this section we assume the expected reward of an arm is given
by an unknown (possibly non-linear) reward function $f:\mathrm{R}^d \rightarrow
\mathrm{R}$:
\begin{align}
r = f(x) + \epsilon \label{eq:kscb_fx}
\end{align}
where $\epsilon$ is a noise term with mean zero. We further assume that $f$ is
from a Reproducing Kernel Hilbert Spaces (RKHS) corresponding to some kernel
$k(\cdot,\cdot)$. We define $\phi:\mathrm{R}^d \rightarrow \mathcal{H}$ as the
mapping from the domain of $x$ to the RKHS $\mathcal{H}$, so that $f(x) =
\langle f, \phi(x) \rangle_{\mathcal{H}}$. In the following we talk about
GP-UCB/CGP-UCB and KernelUCB. GP-UCB/CGP-UCB is a Bayesian approach that puts a
Gaussian Process prior on $f$ to encode the assumption of smoothness, and
KernelUCB is a Frequentist approach that builds estimators from linear
regression in RKHS $\mathcal{H}$ and choose an appropriate regularizer to encode
the assumption of smoothness.
\subsubsection{GP-UCB/CGP-UCB}
The Gaussian Process can be viewed as a prior over a regression function.
\begin{align*}
f(x) \sim GP(\mu(x), k(x, x'))
\end{align*}
where $\mu(x)$ is the mean function and $k(x,x')$ is the covariance function:
\begin{align*}
\mu(x)&=\mathrm{E}(f(x)) \\
k(x,x')&=\mathrm{E}\left((f(x)-\mu(x))(f(x')-\mu(x'))\right)
\end{align*}
Assume the noise term $\epsilon$ in Equation (\ref{eq:kscb_fx}) follows
Gaussian distribution $\mathcal{N}(0, \sigma^2)$ with some variance
$\sigma^2$. Then, given any finite points $\{x_1, ..., x_N\}$, their response
$\boldsymbol{r}_N = [r_1, ..., r_N]^\top$ follows multivariate Gaussian distribution:
\begin{align*}
\boldsymbol{r}_N \sim \mathcal{N}([\mu(x_1), ..., \mu(x_N)]^\top, K_N+\sigma^2\mathrm{I}_N)
\end{align*}
where $(K_N)_{ij} = k(x_i, x_j)$.
It turns out that the posterior distribution of $f$ given $\{x_1, ..., x_N\}$ is
also a Gaussian Process distribution $GP(\mu_N(x), k_N(x, x'))$ with
\begin{align*}
\mu_N(x) &= \boldsymbol{k}_N(x)^\top(K_N+\sigma^2 \mathrm{I})^{-1}\boldsymbol{r}_N \\
k_N(x, x') &=  k(x, x') - \boldsymbol{k}_N(x)^\top(K_N + \sigma^2 \mathrm{I})^{-1}\boldsymbol{k}_N(x')
\end{align*}
where $\boldsymbol{k}_N(x) = [k(x_1, x), ..., k(x_N, x)]^\top$.

GP-UCB~\citep{srinivas2010gaussian} is a Bayesian approach to infer the unknown
reward function $f$. The domain of $f$ is denoted by
$\mathcal{D}$. $\mathcal{D}$ could be a finite set containing $|\mathcal{D}|$
$d-$dimensional vectors, or a infinite set such as $\mathrm{R}^d$. GP-UCB puts a
Gaussian process prior on $f$ : $f \sim GP(\mu(x), k(x, x'))$, and it updates
the posterior distribution of $f$ after each observation. Inspired by the
UCB-style algorithm~\citep{auer2002finite}, it selects an point $x_t$ at time
$t$ with the following strategy:
\begin{align}
x_t = \argmax_{x \in \mathcal{D}} \mu_{t-1}(x) + \sqrt{\beta_t} \sigma_{t-1}(x) \label{eq:kscb_argmax}
\end{align}
where $\mu_{t-1}(x)$ is the posterior mean of $x$, $\sigma_{t-1}^2(x) =
k_{t-1}(x,x)$, and $\beta_t$ is appropriately chosen
constant. (\ref{eq:kscb_argmax}) shows the exploration-exploitation tradeoff of
GP-UCB: large $\mu_{t-1}(x)$ represents high estimated reward, and large
$\sigma_{t-1}(x)$ represents high uncertainty. GP-UCB is described in Algorithm
\ref{algo:gpucb}.
\begin{algorithm}[]
\caption{GP-UCB}
\label{algo:gpucb}
\begin{algorithmic}
\Require $\mu_0=0$, $\sigma_0$, kernel $k$
\For{$t=1,2, ...$}
\State select arm $a_t = \argmax_{a \in \mathcal{A}} \mu_{t-1}(x_{t, a}) + \sqrt{\beta_t} \sigma_{t-1}(x_{t,a})$
\State receive reward $r_{t}$
\State Update posterior distribution of $f$; obtain $\mu_t$ and $\sigma_t$
\EndFor
\end{algorithmic}
\end{algorithm}

The regret of GP-UCB is defined as follow:
\begin{align}
R_T = \sum_{t=1}^T f(x^*) - f(x_t) \label{eq:gpucb_regret}
\end{align}
where $x^* = \argmax_{x \in \mathcal{D}} f(x)$.
From a bandits algorithm's perspective, we can view each data point $x$ in
GP-UCB as an arm; however, in this case the features of an arm won't change
based on the contexts observed, and the best arm is always the same. We can also
view each data point $x$ as a feature vector that encodes both the arm and
context information, however, in that case $x^*$ in Equation
(\ref{eq:gpucb_regret}) becomes $x^* = \argmax_{x \in \mathcal{D}_t} f(x)$ where
$\mathcal{D}_t$ is the domain of $f$ under current context.

Define $\mathrm{I}(\boldsymbol{r}_A; f)=\mathrm{H}(\boldsymbol{r}_A) -
\mathrm{H}(\boldsymbol{r}_A|f)$ as the mutual information between $f$ and
rewards of a set of arms $A \in \mathcal{D}$. Define the maximum information
gain $\gamma_T$ after T rounds as
\begin{align*}
\gamma_T = \max_{A:|A|=T} \mathrm{I}(\boldsymbol{r}_A; f)
\end{align*}
Note that $\gamma_T$ depends on the kernel we
choose. \citet{srinivas2010gaussian} showed that if
$\beta_t=2\ln(Kt^2\pi^2/6\delta)$, then GP-UCB achieves a regret bound of
$\tilde{O}\left( \sqrt{T\gamma_T \ln K} \right)$ with high
probability. \citet{srinivas2010gaussian} also analyzed the agnostic setting,
that is, the true function $f$ is not sampled from a Gaussian Process prior, but
has bounded norm in RKHS:
\begin{theorem}
Suppose the true $f$ is in the RKHS $\mathcal{H}$ corresponding to kernel $k(x,
x')$. Assume $\langle f, f \rangle_{\mathcal{H}} \le B$. Let $\beta_t=2B+300\gamma_t
\ln^3(t/\delta)$, let the prior be $GP(0,k(x, x'))$, and the noise model be
$\mathcal{N}(0, \sigma^2)$. Assume the true noise $\epsilon$ has zero mean
and is bounded by $\sigma$ almost surely. Then the regret bound of GP-UCB is
\begin{align*}
R_T = \tilde{O}\left(\sqrt{T}\left( B\sqrt{\gamma_T}+\gamma_T \right) \right)
\end{align*}
with high probability.
\end{theorem}
\citet{srinivas2010gaussian} also showed the bound of $\gamma_T$ for some common
kernels. For finite dimensional linear kernel $\gamma_T=\tilde{O}(d\ln T)$; for
squared exponential kernel $\gamma_T=\tilde{O}((\ln T)^{d+1})$.

CGP-UCB~\citep{krause2011contextual} extends GP-UCB and explicitly model the
contexts. It defines a context space $\mathcal{Z}$ and an arm space
$\mathcal{D}$; Both $\mathcal{Z}$ and $\mathcal{D}$ can be infinite
sets. CGP-UCB assumes the unknown reward function $f$ is defined over the join
space of contexts and arms:
\begin{align*}
r = f(z, x) + \epsilon
\end{align*}
where $z \in \mathcal{Z}$ and $x \in \mathcal{D}$. The algorithm framework is
the same as GP-UCB except that now we need to choose a kernel k over the joint
space of $\mathcal{Z}$ and $\mathcal{D}$. \citet{krause2011contextual} proposed
one possible kernel $k(\{z, x\}, \{z', x'\}) = k_Z(z, z')k_\mathcal{D}(x,
x')$. We can use different kernels for the context spaces and arm spaces.

\subsubsection{KernelUCB}
KernelUCB~\citep{valko2013finite} is a Frequentist approach to learn the unknown
reward function $f$. It estimates $f$ using regularized linear regression in
RKHS corresponding to some kernel $k(\cdot, \cdot)$. We can also view KernelUCB as a
Kernelized version of LinUCB.

Assume there are $K$ arms in the arm set $\mathcal{A}$, and the best arm at time
$t$ is $a^*=\argmax_{a\in \mathcal{A}}f(x_{t,a})$, then the regret is defined as
\begin{align*}
R_T = \sum_{t=1}^T f(x_{t, a_t^*}) - f(x_{t,a_t})
\end{align*}
We apply kernelized ridge regression to estimate $f$. Given the arms pulled
$\{x_1, ..., x_{t-1}\}$ and
their rewards $\boldsymbol{r}_t=[r_1, ..., r_{t-1}]$ up to time $t-1$,
define the dual variable 
\begin{align*}
\alpha_t = (K_t + \gamma \mathrm{I}_t)^{-1} \boldsymbol{r}_t
\end{align*}
where $(K_t)_{ij} = k(x_i, x_j)$.
Then the predictive value of a given arm $x_{t,a}$ has the following closed form
\begin{align*}
\hat{f}(x_{t,a}) = k_t(x_{t,a})^\top \alpha_t
\end{align*}
where $k_t(x_{t,a}) = [k(x_1, x_{t,a}), ..., k(x_{t-1},x_{t,a})]^\top$. Now we have the
predicted reward, we need to compute the half width of the
confidence interval of the predicted reward. Recall that in LinUCB such half
width is defined as $\sqrt{x_{t,a} (D_t^\top D_t+\gamma\mathrm{I}_d)^{-1}x_{t,a}}$, similarly
in kernelized ridge regression we define the half width as
\begin{align}
\hat{\sigma}_{t,a} = \sqrt{\phi(x_{t,a})^\top(\Phi_t^T\Phi_t+\gamma
  \mathrm{I})^{-1}\phi(x_{t,a})} \label{eq:kucb_width}
\end{align}
where $\phi(\cdot)$ is the mapping from the domain of $x$ to the RKHS, and
$\Phi_t=[\phi(x_1)^\top, ..., \phi(x_{t-1})^\top]^\top$. In order to compute
(\ref{eq:kucb_width}), \citet{valko2013finite} derived a dual representation of
(\ref{eq:kucb_width}):
\begin{align*}
\hat{\sigma}_{t,a} = \gamma^{-1/2} \sqrt{k(x_{t,a}, x_{t,a}) - k_t(x_{t,a})^\top (K_t+\gamma\mathrm{I})^{-1}k_t(x_{t,a})}
\end{align*}
KernelUCB chooses the action $a_t$ at time $t$ with the following strategy
\begin{align*}
a_t = \argmax_{a\in \mathcal{A}}\left(k_t(x_{t,a})^\top \alpha_t + \eta \hat{\sigma}_{t,a}\right)
\end{align*}
where $\eta$ is the scaling parameter.

To derive regret bound, \citet{valko2013finite} proposed SupKernelUCB based on
KernelUCB, which is similar to the relationship between SupLinUCB and
LinUCB. Since the dimension of $\phi(x)$ may be infinite, we cannot directly
apply LinUCB or SupLinUCB's regret bound. Instead, \citet{valko2013finite}
defined a data dependent quantity $\tilde{d}$ called effective dimension: Let
$(\lambda_{i,t})_{i\ge1}$ denote the eigenvalues of $\Phi_t^\top \Phi_t + \gamma
\mathrm{I}$ in decreasing order, define $\tilde{d}$ as
\begin{align*}
  \tilde{d} = \min\{ j:j\gamma \ln T \ge \Lambda_{T, j} \} \text{ where }
  \Lambda_{T, j} = \sum_{i > j} \lambda_{i, T} - \gamma
\end{align*}
$\tilde{d}$ measures how quickly the eigenvalues of $\Phi^\top_t \Phi_t$ are
decreasing. \citet{valko2013finite} showed that if $\sqrt{\langle f,
  f\rangle_{\mathcal{H}}} \le B$ for some $B$ and if we set regularization
parameter $\gamma = 1/B$ and scaling parameter $\eta=\sqrt{2\ln 2TN/\eta}$, then
the regret bound of SupKernelUCB is $\tilde{O}(\sqrt{B\tilde{d}T})$. They showed
that for linear kernel $\tilde{d} \le d$; Also, compared with GP-UCB,
$\mathrm{I}(\boldsymbol{r}_A; f) \ge \Omega(\tilde{d}\ln\ln T)$, which means
KernelUCB achieves better regret bound than GP-UCB in agnostic case.

\subsection{Stochastic Contextual Bandits with Arbitrary Set of Policies}
\subsubsection{Epoch-Greedy}
Epoch-Greedy \citep{langford2008epoch} treats contextual bandits as a classification problem, and it
solves an empirical risk minimization (ERM) problem to find the currently best
policy. One advantage of Epoch-Greedy is that the hypothesis space can be finite
or even infinite with finite VC-dimension, without an assumption of linear payoff.

There are two key problems Epoch-Greedy need to solve in order to achieve low
regret: 1.\ how to get unbiased estimator from ERM; 2.\ how to balance
exploration and exploitation when we don't know the time horizon $T$. To solve
the first problem, Epoch-Greedy makes explicit distinctions between exploration
and exploitation steps. In an exploration step, it selects an arm uniformly at
random, and the goal is to form unbiased samples for learning. In
an exploitation step, it selects the arm based on the best policy learned from
the exploration samples. Of course, Epoch-Greedy adopts the trick we described
in Section \ref{sec:unbiased_estimator_trick} to get unbiased estimator. For the
second problem, note that since Epoch-Greedy strictly separate exploration and 
exploitation steps, so if it already know $T$ in advance then it should always explore
for the first $T'$ steps, and then exploit for the following $T-T'$ steps. The
reason is that there is no advantage to take an exploitation step before the
last exploration step. However generally $T$ is unknown, so Epoch-Greedy
algorithm runs in a mini-batch style: it runs one epoch at a time, and within
that epoch, it first performs one step of exploration, and followed by several
steps of exploitation. The algorithm is shown in Algorithm
\ref{algo:epochgreedy}. 

\begin{algorithm}[!htp]
\caption{Epoch-Greedy}
\label{algo:epochgreedy}
\begin{algorithmic}
  \Require $s(W_\ell)$: exploitation steps given samples $W_\ell$
  \State Init exploration samples $W_0 = \{\}, t_1 = 1$
  \For{$\ell = 1, 2, ...$}
      \State $t = t_\ell$ \quad $\triangleright$ One step of exploration
      \State Draw an arm $a_t \in \{1, ..., K\}$ uniformly at random
      \State Receive reward $r_{a_t} \in [0, 1]$
      \State $W_\ell = W_{\ell - 1} \cup {(x_t, a_t, r_{a_t})}$
      \State Solve $\hat{h}_\ell = \max_{h\in \mathcal{H}} \sum_{(x, a, r_a) \in W_\ell} \frac{r_a \mathds{I}(h(x) = a)}{1/K}$
      \State $t_{\ell + 1} = t_\ell + s(W_\ell) + 1$
      \For{$t=t_\ell+1, ..., t_\ell - 1$} \quad $\triangleright$ $s(W_\ell)$ steps of exploration
          \State Select arm $a_t = \hat{h}_\ell(x_t)$
          \State Receive reward $r_{a_t} \in [0, 1]$
      \EndFor
  \EndFor
\end{algorithmic}
\end{algorithm}

Different from the EXP4 setting, we do not assume an adversary environment
here. Instead, we assume there is a distribution $P$ over $(x, \ve{r})$,
where $x \in \mathcal{X}$ is the context and $\ve{r} = [r_1, ..., r_K] \in [0, 1]^K$ is the reward
vector. At time $t$, the world reveals context $x_t$, and the algorithm selects
arm $a_t \in \{1, ... K\}$ based on the context, and then the world reveals the
reward $r_{a_t}$ of arm $a_t$. 
The algorithm makes its decision based on a policy/hypothesis $h \in \mathcal{H}:
\mathcal{X} \rightarrow \{1, ..., K\}$. $\mathcal{H}$ is the policy/hypothesis
space, and it can be an infinite space such as all linear hypothesis in dimension
$d$, or it can be a finite space consists of $N = |\mathcal{H}|$ hypothesis. In this
survey we mainly focus on finite space, but it is easy
to extend to infinite space.

Let $Z_t = (x_t, a_t, r_{a_t})$ be the $t^{th}$ exploration sample, and $Z_1^n
= \{Z_1, ..., Z_n\}$. The expected reward of a hypothesis $h$ is
\begin{align*}
R(h) = \mathrm{E}_{(x,\ve{r}) \sim P} [r_{h(x)}]
\end{align*}
so the regret of the algorithm is
\begin{align*}
R_T = \sup_{h\in \mathcal{H}}TR(h) - \mathrm{E} \sum_{t=1}^T r_{a_t}
\end{align*}
The expectation is with respect to $Z_1^n$ and any random variable in the
algorithm.

Denote the data-dependent exploitation step count by
$s(Z_1^n)$, so $s(Z_1^n)$ means that based on all samples $Z_1^n$ from
exploration steps, the algorithm
should do $s(Z_1^n)$ steps exploitation. The hypothesis that maximizing the
empirical reward is
\begin{align*}
\hat{h}(Z_1^n) = \argmax_{h\in \mathcal{H}} \sum_{t=1}^n \frac{r_{a_t} \mathds{I}(h(x_t) = a_t)}{1/K}
\end{align*}
The per-epoch exploitation cost is
defined as 
\begin{align*}
\mu_n(\mathcal{H}, s) &= \mathrm{E}_{Z_1^n}\left(\sup_{h\in \mathcal{H}}R(h) -
  R(\hat{h}(Z_1^n)) \right) s(Z_1^n)
\end{align*}
When $S(Z_1^n) = 1$
\begin{align*}
\mu_n(\mathcal{H}, 1) &= \mathrm{E}_{Z_1^n}\left(\sup_{h\in \mathcal{H}}R(h) -
  R(\hat{h}(Z_1^n)) \right) 
\end{align*}
The per-epoch exploration regret is less or equal to 1 since we only do one step
exploration, so we would want to select a $s(Z_1^n)$ such that the per-epoch exploitation regret
$\mu_n(\mathcal{H}, s) = 1$. Later we will show how to choose $s(Z_1^n)$.
\begin{theorem}
  \label{thm:epochgreedy1}
For all $T$, $n_\ell$, $L$ such that: $T \le L + \sum_{\ell = 1}^L n_\ell$, the
regret of Epoch-Greedy is bounded by
\begin{align*}
R_T \le L + \sum_{\ell=1}^L \mu_\ell (\mathcal{H}, s) + T\sum_{\ell=1}^L
  P[s(Z_1^\ell) < n_\ell]
\end{align*} 
\end{theorem}
The above theorem means that suppose we only consider the first $L$ epochs, and for
each epoch $\ell$, we use a sample independent variable $n_\ell$ to bound
$S(Z_1^\ell)$, then the regret up to time $T$ is bounded by the above.
\begin{proof}
based on the relationship between $s(Z_1^n)$ and $n_\ell$, one of the following
two events will occur:
\begin{enumerate}
\item $s(Z_1^\ell) < n_\ell$ for some $\ell = 1,..., L$
\item $s(Z_1^\ell) \ge n_\ell$ for all $\ell = 1,...,L$
\end{enumerate}
In the second event, $n_\ell$ is the lower bound for $s(Z_1^\ell)$, so $T \le L
+ \sum_{\ell=1}^L n_\ell \le \L + \sum_{\ell=1}^L s(Z_1^\ell)$, so the epoch
that contains $T$ must be less or equal to epoch $L$, hence the regret is less
or equal to the sum of the regret in the first $L$ epochs. Also within each
epoch, the algorithm do one step exploration and then $s_{Z_1^\ell}$ exploitation, so
the regret bound when event 2 occurs is
\begin{align*}
R_{T, 2} \le L + \sum_{\ell=1}^L \mu_\ell(\mathcal{H}, s)
\end{align*}
The regret bound when event 1 occurs is $R_{T, 1} \le T$ because the reward $r \in
[0, 1]$. Together we get the regret bound
\begin{align*}
R_T &\le T \sum_{\ell=1}^L P[s(Z_1^\ell) < n_\ell] + \prod_{\ell=1}^L
      P[s(Z_1^\ell) \ge n_\ell] \sum_{\ell=1}^L (1 +
  \mu_\ell(\mathcal{H}, s))  \\
& \le T \sum_{\ell=1}^L P[s(Z_1^\ell) < n_\ell] + L + \sum_{\ell=1}^L
  \mu_\ell(\mathcal{H}, s)
\end{align*}
\end{proof}

Theorem \ref{thm:epochgreedy1} gives us a general bound, we now derive a
specific problem-independent bound based on that.

One essential thing we need to do is to bound $\sup_{h \in \mathcal{H}} R(h) - R(\hat{h}(Z_1^n))$.
If hypothesis space $\mathcal{H}$ is finite, we can use finite class uniform
bound, and if $\mathcal{H}$ is infinite, we can use VC-dimension or other
infinite uniform bound techniques. The two proofs are similar, and here to
consistent with the original paper, we assume $\mathcal{H}$ is a finite space.

\begin{theorem}[Bernstein]
If $P(|Y_i| \le c) = 1$ and $\mathrm{E}(Y_i) = 0$, then for any $t > 0$,
\begin{align*}
P(|\overline{Y}_n| > \epsilon) \le 2\exp\left\{ -\frac{n\epsilon^2}{2\sigma^2+2c\epsilon/3} \right\}
\end{align*}
where $\sigma^2 = \frac{1}{n}\sum_{i=1}^n Var(Y_i)$.
\end{theorem}
\begin{theorem}
With probability $1-\delta$, the problem-independent regret of Epoch-Greedy is
\begin{align*}
R_T \le c T^{2/3}(K \ln (|\mathcal{H}|/\delta))^{1/3}
\end{align*}
\end{theorem}
\begin{proof}
Follow Section \ref{sec:unbiased_estimator_trick}, define $\hat{R}(h) = \frac{1}{n}\sum_i \frac{\mathds{I}(h(x_i) = a_i)r_{a_i}}{1/K}$, the empirical sample reward of a
hypothesis $h$. Also define $\hat{R}_i =
\frac{\mathds{I}(h(x_i)=a_i)r_{a_i}}{1/K}$, then $\mathrm{E}\hat{R}(h) =
R(h)$, and
\begin{align*}
\mathrm{var}(\hat{R}_i) &\le \mathrm{E} (\hat{R}_i^2) \\
&= \mathrm{E} K^2\mathds{I}(h(x_i) = a_i) r_{a_i}^2 \\
&\le \mathrm{E} K^2\mathds{I}(h(x_i) = a_i) \\
&= \mathrm{E} K^2 1/K \\
& = K 
\end{align*}
So the variance is bounded by K and we can apply Bernstein
inequality to get:
\begin{align*}
P(|\hat{R}(h) - R(h)| > \epsilon) \le 2\exp\left\{
  -\frac{n\epsilon^2}{2K+2c\epsilon/3} \right\}
\end{align*}
From union bound we have
\begin{align*}
P\left(\sup_{h\in \mathcal{H}}|\hat{R}(h) - R(h)| > \epsilon \right) \le 2N\exp\left\{
-\frac{n\epsilon^2}{2K+2c\epsilon/3} \right\}
\end{align*}
Set the right-hand side to $\delta$ and solve for $\epsilon$ we have,
\begin{align*}
\epsilon = c\sqrt{\frac{K\ln(N/\delta)}{n}}
\end{align*}
So, with probability $1-\delta$,
\begin{align*}
\sup_{h\in \mathcal{H}}|\hat{R}(h) - R(h)| \le c\sqrt{\frac{K\ln(N/\delta)}{n}}
\end{align*}
Let $\hat{h}$ be the estimated hypothesis, and $h_*$ be the best hypothesis, then with probability $1-\delta$,
\begin{align*}
R(\hat{h}) \le \hat{R}(\hat{h}) + c\sqrt{\frac{K\ln(N/\delta)}{n}} \le
  \hat{R}(h_*) + c\sqrt{\frac{K\ln(N/\delta)}{n}} \le R(h_*) + 2c\sqrt{\frac{K\ln(N/\delta)}{n}}
\end{align*}
So \begin{align*}\mu_\ell(\mathcal{H}, 1) \le 2c
     \sqrt{\frac{K\ln(N/\delta)}{\ell}}\end{align*}
To make $\mu_\ell(\mathcal{H}, s) \le 1$, we can choose 
\begin{align*}s(Z_1^\ell) = \lfloor
c'\sqrt{\ell/(K \ln (N/\delta))} \rfloor\end{align*}
Take $n_\ell = \lfloor c'\sqrt{\ell/(K \ln (N/\delta))} \rfloor$, then
$P[s(Z_1^\ell) < n_\ell] = 0$. So the regret 
\begin{align*}
R_T &\le L + \sum_{\ell=1}^L \mu_\ell(\mathcal{H}, s) \\ 
&\le 2L
\end{align*}
Now the only job is to find the $L$. We can pick a $L$ such that $T \le
\sum_{\ell=1}^L n_\ell$, so $T$ will also satisfy $T \le L + \sum_{\ell=1}^L n_\ell$.
\begin{align*}
T &= \sum_{\ell=1}^L n_\ell \\
&= \sum_{\ell=1}^L \lfloor c' \sqrt{\ell/K\ln (N/\delta)} \rfloor \\
&= c'\lfloor \sqrt{1/(K\ln (N/\delta))} (\sum_{\ell=1}^L \sqrt{\ell}) \rfloor \\
&= c''\lfloor \sqrt{1/(K\ln (N/\delta))} L^{3/2} \rfloor
\end{align*}
So 
\begin{align*}
L = c'' \lfloor (K\ln (N/\delta))^{1/3} T^{2/3} \rfloor \\
R_T  \le  c'''(K\ln (N/\delta))^{1/3} T^{2/3}
\end{align*}
Hence, with probability $1-\delta$, the regret of Epoch-Greedy is $O((K\ln(
N/\delta))^{1/3} T^{2/3})$.
\end{proof}
Compared to EXP4, Epoch-Greedy has a weaker bound but it converge with
probability instead of expectation; Compared to EXP4.P, Epoch-Greedy has a
weaker bound but it does not require the knowledge of $T$.

\subsubsection{RandomizedUCB}
Recall that in EXP4.P and Epoch-Greedy we are always competing with the best
policy/expert, and the optimal regret bound $O(\sqrt{KT\ln N})$ scales only
logarithmically in the number of policies, so we could boost the model
performance by adding more and more potential policies to the policy set $\mathcal{H}$. With
high probability EXP4.P achieves the optimal regret, however the running time
scales linearly instead of logarithmically in the number of experts. As a
results, we are constrained by the computational bottleneck. Epoch-Greedy could
achieve sub-linear running time depending on what assumptions we make about the
$\mathcal{H}$ and ERM, however the regret bound is $O((K\ln N)^{(1/3)}
T^{(2/3)})$, which is sub-optimal. RandomizedUCB \citep{dudik2011efficient}, on the other hand, could
achieve optimal regret while having a polylog(N) running time. One key
difference compared to Epoch-Greedy is that it assigns a non-uniform
distribution over policies, while Epoch-Greedy assigns uniform distribution when
doing exploration. Also RandomizedUCB does not make explicit distinctions
between exploration and exploitation. 

Similar to Epoch-Greedy, let $\mathcal{A}$ be a set of K arms $\{1, ..., K\}$,
and $D$ be an arbitrary distribution over $(x, \ve{r})$, where $x \in
\mathcal{X}$ is the context and $\ve{r} \in [0, 1]^K$ is the reward
vector. Let $D_X$ be the marginal distribution of $D$ over $x$.
At time $t$, the world samples a $(x_t, \ve{r}_t)$ pair and reveals $x_t$ to
the algorithm, the algorithm then picks an arm $a_t \in \mathcal{A}$ and then receives reward
$r_{a_t}$ from the world.
Denote a set of policies ${h: \mathcal{X} \rightarrow \mathcal{A}}$ by
$\mathcal{H}$. The algorithm has access to $\mathcal{H}$ and makes decisions
based on $x_t$ and $\mathcal{H}$. The expected reward of a policy $h \in
\mathcal{H}$ is
\begin{align*}
R(h) = \mathrm{E}_{(x, \ve{r}) \sim D}[r_{h(x)}] 
\end{align*}
and the regret is defined as
\begin{align*}
R_T = \sup_{h\in \mathcal{H}}TR(h) - \mathrm{E} \sum_{t=1}^T r_{a_t}
\end{align*}
Denote the sample at time $t$ by $Z_t = (x_t, a_t, r_{a_t}, p_{a_t})$, where
$p_{a_t}$ is the probability of choosing $a_t$ at time $t$. Denote all the
samples up to time $t$ by $Z_1^t = \{Z_1, ..., Z_t\}$. Then the unbiased
reward estimator of policy $h$ is
\begin{align*}
\hat{R}(h) = \frac{1}{t} \sum_{(x, a, r, p) \in Z_1^t} \frac{r \mathds{I}(h(x) = a)}{p}
\end{align*} 
The unbiased empirical reward maximization estimator at time $t$ is
\begin{align*}
\hat{h}_t = \argmax_{h\in \mathcal{H}} \sum_{(x, a, r, p) \in Z_1^t} \frac{r \mathds{I}(h(x) = a)}{p}
\end{align*}
RandomizedUCB chooses a distribution $P$ over policies $\mathcal{H}$ which in
turn induce distributions over arms. Define
\begin{align*}
W_P(x, a) = \sum_{h(x) = a} P(h)
\end{align*}
be the induced distribution over arms, and
\begin{align*}
W'_{P, \mu}(x, a) = (1-K\mu)W_P(x, a) + \mu
\end{align*}
be the smoothed version of $W_P$ with a minimum probability of $\mu$. Define 
\begin{align*}
R(W) &= \mathrm{E}_{(x, \ve{r})\sim D} [\ve{r} \cdot W(x)] \\
\hat{R}(W) &= \frac{1}{t} \sum_{(x, a, r, p) \in Z_1^t}\frac{rW(x, a)}{p}
\end{align*}
To introduce RandomizedUCB, let's introduce POLICYELIMINATION algorithm
first. POLICYELIMINATION is not practical but it captures the basic ideas behind
RandomizedUCB. The general idea is to find the best policy by empirical
risk. However empirical risk suffers from variance (no bias since we again adopt
the trick in Section \ref{sec:unbiased_estimator_trick}), so POLICYELIMINATION
chooses a distribution $P_t$ over all policies to control the variance of
$\hat{R}(h)$ for all policies, and then eliminate policies that are not likely to be
optimal. 

\begin{algorithm}[!htp]
\caption{POLICYELIMINATION}
\label{algo:policyelimination}
\begin{algorithmic}
\Require $\delta \in (0, 1]$
\State Define $\delta_t = \delta / 4Nt^2$, $b_t = 2\sqrt{\frac{2K\ln
    (1/\delta_t)}{t}}, \mu_t = \min\left\{ \frac{1}{2K}, \sqrt{\frac{\ln(1/\delta_t)}{2Kt}} \right\}$
\For{t=1,..., T}
\State Choose a distribution $P_t$ over $\mathcal{H}_{t-1}$ s.t.\ $\forall \ h \in
\mathcal{H}_{t-1}$
\begin{align}
\mathrm{E}_{x \in D_X} \left[ \frac{1}{W'_{P_t, \mu_t}(x, h(x))} \right] \le 2K
\label{algo:policyelimination:eq1}
\end{align}
\State Sample $a_t$ from $W'_t = W'_{P_t, \mu_t}(x_t, \cdot)$
\State Receive reward $r_{a_t}$
\State Let 
\begin{align}
\mathcal{H}_{t} = \left\{ h \in \mathcal{H}_{t-1} : \delta_t(h) \ge \left(\max_{h'\in
  \mathcal{H}_{t-1}} \delta_t(h')\right) - 2b_t \right\}
  \label{algo:policyelimination:eqelim}
\end{align}
\EndFor
\end{algorithmic}
\end{algorithm}
By Minimax theorem \citet{dudik2011efficient} proved that there always exists a
distribution $P_t$ satisfy the constrain in Algorithm
\ref{algo:policyelimination}.

\begin{theorem}[Freedman-style Inequality] \label{thm:freedman} Let $y_1, ..., y_T$ be a sequence of
  real-valued random variables. Let $V, R \in \mathrm{R}$ such that
  $\sum_{t=1}^T \mathrm{var}[y_t] \le V$, and for all $t$, $y_t -
  \mathrm{E}_t[y_t] \le R$. Then for any $\delta > 0$ such that $R \le
  \sqrt{V/\ln(2/\delta)}$, with probability at least $1-\delta$,
\begin{align*}
\left| \sum_{t=1}^T y_t - \sum_{t=1}^T\mathrm{E}_t[y_t] \right| \le 2 \sqrt{V \ln(2/\delta)}
\end{align*}
\end{theorem}
\begin{theorem}
\label{thm:policyeliminationregret}
With probability at least $1-\delta$, the regret of POLICYELIMINATION is bounded
by:
\begin{align*}
R_T = O(16\sqrt{2TK\ln \frac{4T^2N}{\delta}})
\end{align*}
\end{theorem}
\begin{proof}
Let \begin{align*}\hat{R}_i(h) =
      \frac{r_t\mathds{I}(h(x_t)=a_t)}{W'_t(h(x_t))}\end{align*}
the estimated reward of policy $h$ at time $t$. To make use of Freedman's
inequality, we need to bound the variance of $\hat{R}_i(h)$
\begin{align*}
\mathrm{var}(\hat{R}_i(h)) &\le \mathrm{E} \hat{R}_i(h)^2 \\
&=\mathrm{E} \frac{r_t^2\mathds{I}(h(x_t) = a_t)}{W'_t(h(x_t))^2} \\
&\le \mathrm{E} \frac{\mathds{I}(h(x_t) = a_t)}{W'_t(h(x_t))^2} \\
&= \mathrm{E} \frac{1}{W'_t(h(x_t))} \\
&\le 2K
\end{align*}
The last inequality is from the constrain in Equation\
(\ref{algo:policyelimination:eq1}).
So 
\begin{align*}
\sum_{i=1}^t \mathrm{var}[\hat{R}_i(h)^2] \le 2Kt = V_t
\end{align*}
Now we need to check if $R_t$ satisfy the constrain in Theorem
\ref{thm:freedman}. Let $t_0$ be the first $t$ such that $\mu_t < 1/2 K$. when $t
\ge t_0$, then for all $t' \le t$,
\begin{align*}
\hat{R}_{t'}(h) \le 1/\mu_{t'} \le 1/\mu_t = \sqrt{\frac{2Kt}{\ln (1/\delta_t)}} =
  \sqrt{\frac{V_t}{\ln (1/\delta_t)}}
\end{align*}
So now we can apply Freedman's inequality and get
\begin{align*}
P(|\hat{R}(h) - R(h)| \ge b_t) \le 2\delta_t
\end{align*}
Take the union bound over all policies and $t$
\begin{align*}
  \sup_{t'\in t}\sup_{h\in \mathcal{H}} P(|\hat{R}(h) - R(h)| \ge b_t) &\le
  2N\sum_{t'=1}^t \delta_{t'} \\
&\le \sum_{t'=1}^t \frac{\delta}{2t'^2} \\
&\le \delta
\end{align*}
So with probability $1-\delta$, we have
\begin{align*}
|\hat{R}(h) - R(h)| \le b_t
\end{align*}
When $t < t_0$, then $\mu_t < 1/2K$ and $b_t \ge 4K\mu_t \ge 2$, then the above
bound still holds since reward is bounded by 1.

To sum up, we make use of the convergence of $\sum_{t} \frac{1}{t^2}$ to
construct $\delta_t$ so that the union bound is less than $\delta$, and we use $R_t$'s
constrain in Freedman's inequality to construct $u_t$ and Freedman's inequality
to construct $b_t$.
\begin{lemma}
\label{lemma:policyelimination1}
With probability at least $1-\delta$,
\begin{align*}
  |\hat{R}(h) - R(h)| \le b_t
\end{align*}
\end{lemma}

\noindent From Lemma \ref{lemma:policyelimination1} we have
\begin{align*}
 \hat{R}(h) - b_t \le R(h) &\le R(h^*) \le \hat{R}(h^*) + b_t \\
\hat{R}(h) &\le \hat{R}(h^*) + 2b_t \\
\end{align*}
where $h^* = \max_{h \in \mathcal{H}} R(h^*)$.
So we can see that $h^*$ is always in $\mathcal{H}_t$ after the policy
elimination step (Equation\ \ref{algo:policyelimination:eqelim}) in Algorithm \ref{algo:policyelimination}.
Also, if $R(h) \le R(h^*) - 4b_t$, then
\begin{align*}
\hat{R}(h) - b_t \le R(h) &\le R(h^*) - 4b_t \le \hat{R}(h^*) + b_t - 4b_t \\
\hat{R}(h) &\le \hat{R}(h^*) - 2b_t
\end{align*}
However, as we can see from the elimination step, all the policies which satisfy
$\hat{R}(h) \le \hat{R}(h^*) - 2b_t$ is eliminated. So for all the remaining
policies $h \in \mathcal{H}_{t}$, we have $R(h^*) - R(h) \le 4b_t$, so the
regret
\begin{align*}
  R_T &\le \sum_{t=1}^T R(h^*) - R(h) \\
&\le 4\sum_{t=1}^T b_t \\
&\le 8\sqrt{2K\ln \frac{4Nt^2}{\delta}} \sum_{t=1}^T\frac{1}{\sqrt{t}} \\
&\le 8\sqrt{2K\ln \frac{4Nt^2}{\delta}} 2 \sqrt{T} \\
&\le 16 \sqrt{2TK\ln \frac{4NT^2}{\delta}}
\end{align*}
\end{proof}

POLICYELIMINATION describes the basic idea of RandomizedUCB, however
POLICYELIMINATION is not practical because it does not actually show how to find
the distribution $P_t$, also it requires the knowledge of $D_x$. To solve these
problems, RandomzedUCB always considers the full set of policies and use an
argmax oracle to find the distribution $P_t$ over all policies, and instead of using
$D_x$, the algorithm uses history samples. Define
\begin{align*}
\Delta_D(W) = R(h^*) - R(W) \\
\Delta_t(W) = \hat{R}(\hat{h}_t) - \hat{R}(W)
\end{align*}
RandomizedUCB is described in
Algorithm \ref{algo:randomizeducb}.
\begin{algorithm}[!htp]
\caption{RandomizedUCB}
\label{algo:randomizeducb}
\begin{algorithmic}
\State Define $W_0 = \{\}$, $C_t=2\ln (\frac{Nt}{\delta})$, $\mu_t = \min\left\{ \frac{1}{2K}, \sqrt{\frac{C_t}{2Kt}} \right\}$
\For{t=1,..., T}
\State Solve the following optimization problem to get distribution $P_t$ over
$\mathcal{H}$
\begin{align*}
\min_P \sum_{h\in \mathcal{H}} P(h) \Delta_{t-1}(h)
\end{align*}
\State \quad s.t. for all distribution $Q$ over $\mathcal{H}$:
\begin{align}
\mathrm{E}_{h \sim Q} \left[ \frac{1}{t-1}\sum_{i=1}^{t-1} \frac{1}{W'_{P_t,
  \mu_t}(x, h(x))} \right] \le \max\left\{ 4K,
  \frac{(t-1)\Delta_{t-1}(W_Q)^2}{180C_{t-1}} \right\}
\label{algo:randomizeducb:eqst}
\end{align}
\State Sample $a_t$ from $W'_t = W'_{P_t, \mu_t}(x_t, \cdot)$
\State Receive reward $r_{a_t}$
\State $W_t = W_{t-1} \cup (x_t, a_t, r_{a_t}, W'_t(a_t))$
\EndFor
\end{algorithmic}
\end{algorithm}
Similar to POLICYELIMINATION, $P_t$ in RandomizedUCB algorithm is to control
the variance. However, instead of controlling each policy separately, it
controls the expectation of the variance with respect to the distribution
$Q$. The right-hand side of Equation\ (\ref{algo:randomizeducb:eqst}) is upper bounded
by $c\Delta_{t-1}(W_Q)^2$, which measures the empirical performance of
distribution $Q$. So the general idea of this optimization problem is to bound
the expected variance of empirical reward with respect to all possible
distribution $Q$, whereas if $Q$ achieves high empirical reward then the bound is
tight hence the variance is tight, and if $Q$ has low empirical reward, the
bound is loose. This makes sure that $P_t$ puts more weight on policies with low
regret. \citet{dudik2011efficient} showed that the regret of RandomizedUCB is
$O(\sqrt{TK\ln(TN/\delta)})$.

To solve the optimization problem in the algorithm, RandomizedUCB uses an
argmax oracle($\mathcal{AMO}$) and relies on the ellipsoid method. The main
contribution is the following theorem:
\begin{theorem}
In each time $t$ RandomizedUCB makes $O(t^5K^4\ln^2(\frac{tK}{\delta}))$ calls to
$\mathcal{AMO}$, and requires additional $O(t^2K^2)$ processing time. The total
running time at each time $t$ is $O(t^5K^4\ln^2(\frac{tK}{\delta})\ln N)$, which is sub-linear.
\end{theorem}

\subsubsection{ILOVETOCONBANDITS}
\noindent (need more details) \\
Similar to RandomizedUCB, Importance-weighted LOw-Variance Epoch-Timed
Oracleized CONtextual BANDITS algorithm (ILOVETOCONBANDITS) proposed by \citet{agarwal2014taming} aims to run in time sub-linear with
respect to $N$ (total number of policies) and achieves optimal regret bound $O(\sqrt{KT\ln
  N})$. RandomizedUCB makes $O(T^6)$ calls to $\mathcal{AMO}$ over all
$T$ steps,  and ILOVETOCONBANDITS tries to further reduce this time complexity.
\begin{theorem}
ILOVETOCONBANDITS achieves optimal regret bound, requiring
$\tilde{O}(\sqrt{\frac{KT}{\ln(N/\delta)}})$ calls to $\mathcal{AMO}$ over $T$
rounds, with probability at least $1-\delta$.
\end{theorem}
Let $\mathcal{A}$ be a finite set of $K$ actions, $x \in \mathcal{X}$ be a
possible contexts, and $\ve{r} \in [0, 1]^K$ be the reward vector of arms
in $\mathcal{A}$. We assume $(x, \ve{r})$ follows a distribution $\mathcal{D}$. Let
$\Pi$ be a finite set of policies that map contexts $x$ to actions $a\in
\mathcal{A}$, let $Q$ be a distribution over all policies $\Pi$, and $\Delta^\Pi$ be the set
of all possible $Q$. ILOVETOCONBANDITS is described in Algorithm
\ref{algo:ilovetoconbandits}. The $Sample(x_t, Q_{m-1}, \pi_{\tau_m-1},
\mu_{m-1})$ function is described in Algorithm
\ref{algo:ilovetoconbanditssample}, it samples an action from a sparse
distribution over policies.

\begin{algorithm}[!htp]
\caption{ILOVETOCONBANDITS}
\label{algo:ilovetoconbandits}
\begin{algorithmic}
  \Require Epoch schedule $0=\tau_0 < \tau_1 < \tau_2 < ...$, $\delta \in (0, 1)$
  \State Initial weights $Q_0 = \ve{0}$, $m=1$, $\mu_m=\min\{\frac{1}{2K}, \sqrt{\ln(16\tau_m^2|\Pi|/\delta)/(K\tau_m)}\}$
  \For{$t=1, 2, ..., T$}
    \State $(a_t, p_t(a_t)) = Sample(x_t, Q_{m-1}, \pi_{\tau_m-1}, \mu_{m-1})$
    \State Pull arm $a_t$ and receive reward $r_t \in [0, 1]$
    \If{$t = \tau_m$}
      \State Let $Q_m$ be a solution to (OP) with history $H_t$ and minimum probability $\mu_m$
      \State $m=m+1$
    \EndIf
  \EndFor
\end{algorithmic}
\end{algorithm}

\begin{algorithm}[!htp]
\caption{Sample}
\label{algo:ilovetoconbanditssample}
\begin{algorithmic}
  \Require $x, Q, \mu$
  \For{$\pi \in \Pi$ and $Q(\pi) > 0$}
  \State $p_{\pi(x)} = (1-K\mu)Q(\pi) + \mu$
  \EndFor
  \State Randomly draw action $a$ from $p$
  \State return $(a, p_{a})$
\end{algorithmic}
\end{algorithm}

As we can see, the main procedure of ILOVETOCONBANDITS is simple. It solves an
optimization problem on pre-specified rounds $\tau_1, \tau_2, ...$ to get a
sparse distribution $Q$ over all policies, then it samples an action based on
this distribution. The main problem now is to choose an sparse distribution $Q$
that achieves low regret and requires calls to $\mathcal{AMO}$ as little as
possible. 

Let $\hat{R}_t(\pi)$ be the unbiased reward estimator of policy $\pi$ over the
first $t$ rounds (see section \ref{sec:unbiased_estimator_trick}), and let $\pi_t = \argmax_\pi \hat{R}_t(\pi)$, then the
estimated empirical regret of $\pi$ is $\widehat{Reg}_t(\pi) = \hat{R}_t(\pi_t) - \hat{R}(\pi)$.
Given a history $H_t$ and minimum probability $\mu_m$, and define $b_\pi =
\frac{\widehat{Reg}_t(\pi)}{\psi \mu_m}$ for $\psi = 100$, then the optimization problem is
to find a distribution $Q \in \Delta^\Pi$ such that
\begin{gather}
\sum_{\pi \in \Pi} Q(\pi)b_\pi \le 2K \label{eq:ilovetoconbandit1}\\
\forall \pi \in \Pi: \mathrm{E}_{x\sim H_t}\left[ \frac{1}{Q^{\mu_m}(\pi(x)|x)}
  \le 2K + b_\pi \right] \label{eq:ilovetoconbandit2}
\end{gather}
where $Q^{\mu_m}$ is the smoothed version of $Q$ with minimum probability
$\mu_m$.

Note that $b_\pi$ is a scaled version of empirical regret of $\pi$, so
Equation\ (\ref{eq:ilovetoconbandit1}) is actually a bound on the expected empirical regret
with respect to $Q$. This equation can be treated as the exploitation since we want to
choose a distribution that has low empirical regret. Equation\
(\ref{eq:ilovetoconbandit2}), similar to RandomizedUCB, is a bound on the
variance of the reward estimator of each policy $\pi \in \Pi$. If the policy has low
empirical regret, we want it to have smaller variance so that the reward estimator is more
accurate, on the other hand, if the policy has high empirical regret, then we allow it to have a
larger variance.

\citet{agarwal2014taming} showed that this optimization problem can be solved via
coordinate descent with at most $\tilde{O}(\sqrt{Kt/\ln(N/\delta)})$ calls to
$\mathcal{AMO}$ in round $t$, moreover, the support (non-zeros) of the resulting
distribution $Q$ at time $t$ is at most $\tilde{O}(\sqrt{Kt/\ln(N/\delta)})$
policies, which is the same as the number of calls to $\mathcal{AMO}$. This
results sparse $Q$ and hence sub-linear time complexity for $Sample$
procedure.

\citet{agarwal2014taming} also showed that the requirement of $\tau$
is that $\tau_{m+1}- \tau_m = O(\tau_m)$. So we can set $\tau_m = 2^{m-1}$, then
the total number of calls to $\mathcal{AMO}$ over all $T$ round is only
$\tilde{O}(\sqrt{Kt/\ln(N/\delta)})$, which is a vast improvement over RandomizedUCB.

\begin{theorem}
With probability at least $1-\delta$, the regret of ILOVETOCONBANDITS is
\begin{align*}
O(\sqrt{KT\ln(TN/\delta)} + K\ln(TN/\delta))
\end{align*}
\end{theorem}

\section{Adversarial Contextual Bandits}
In adversarial contextual bandits, the reward of each arm does not necessarily
follow a fixed probability distribution, and it can be picked by an adversary
against the agent. One way to solve adversarial contextual bandits problem is to model it with expert
advice. In this method, there are $N$ experts, and at each time $t$ each expert gives advice
about which arm to pull based on the contexts. The agent has its own strategy to
pick an arm based on all the advice it gets. Upon receiving the reward of
that arm, the agent may adjust its strategy such as changing the weight or believe of
each expert. 
\subsection{EXP4}
\label{sec:exp4}
Exponential-weight Algorithm for Exploration and Exploitation using Expert
advice (EXP4) \citet{auer2002nonstochastic} assumes each
expert generates an
advice vector based on the current context $x_t$ at time $t$. Advice vectors are
distributions over arms, and are denoted by $\ve{\xi^1_t}, \ve{\xi^2_t}, ..., \ve{\xi^N_t} \in [0,
1]^K$. $\xi^i_{t, j}$ indicates expert $i$'s recommended probability of playing
arm $j$ at time $t$. The algorithm pulls an arm based on these
advice vectors. Let $\ve{r_t} \in [0, 1]^K$ be the true reward vector at time $t$, then
the expected reward of expert $i$ is $\ve{\xi^i_t} \cdot \ve{r_t}$. The algorithm
competes with the best expert, which achieves the highest expected cumulative reward
\begin{align*}
G_{max} = \max_i \sum_{t=1}^T \ve{\xi^i_t} \cdot \ve{r_t}
\end{align*}
The regret is defined as:
\begin{align*}
R_T = \max_i \sum_{t=1}^T \ve{\xi^i_t} \cdot \ve{r_t} - \mathrm{E} \sum_{t=1}^T
  r_{t, a_t}
\end{align*}
The expectation is with respect to the algorithm's random choice of the arm and
any other random variable in the algorithm. Note that we don't have any
assumption on the distribution of the reward, so EXP4 is a adversarial bandits algorithm. 

EXP4 algorithm is described in Algorithm \ref{algo:exp4}. Note that the context $x_t$
does not appear in the algorithm, since it is only used by experts
to generate advice.
\begin{algorithm}[!htp]
\caption{EXP4}
\label{algo:exp4}
\begin{algorithmic}
  \Require $\gamma \in (0, 1]$
  \State Set $w_{t, i} = 1$ for $i = 1, ..., N$
  \For{$t = 1, 2, ..., T$}
  \State Get expert advice vectors $\{\ve{\xi}^1_t, ..., \ve{\xi}^N_t\}$, each vector is a distribution over arms.
  \For{$j=1, ..., K$}
  \begin{align*}
    p_{t, j} = (1-\gamma) \sum_{i=1}^N \frac{w_{t, i} \xi_{t, j}^{
    i}}{\sum_{i=1}^N w_{t, i}} + \frac{\gamma}{K}
\end{align*}
  \EndFor
  \State Draw action $a_t$ according to $p_t$, and receive reward $r_{a_t}$.
  \For{$j=1, ..., K$}  \quad $\triangleright$ Calculate unbiased estimator of $r_t$
  \begin{align*}
    \hat{r}_{t, j} = \frac{r_{t, j}}{p_{t, j}} \mathds{I} (j = a_t)
  \end{align*}
  \EndFor
  \For{$i=1, ..., N$} \quad $\triangleright$ Calculate estimated expected reward and update weight
\begin{align*} 
  \hat{y}_{t, i} &= \ve{\xi}^i_t \cdot \ve{\hat{r}}_{t} \\
  w_{t+1, i} &= w_{t} \exp(\gamma \hat{y}_{t, i} / K)
  \end{align*}
  \EndFor
  \EndFor
\end{algorithmic}
\end{algorithm}

If an expert assigns uniform weight to all actions in each time $t$, then we call the expert a
uniform expert.
\begin{theorem}
For any family of experts which includes a uniform expert, EXP4's regret is
bounded by $O(\sqrt{TK\ln N})$.
\end{theorem}
\begin{proof}
The general idea of the proof is to bound the expected cumulative
reward $\mathrm{E}\sum_{t=1}^T r_{t, a_t}$, then since $G_{max}$ is bounded by the
time horizon T, we can get a bound on $G_{max} - \mathrm{E}\sum_{t=1}^T r_{t, a_t}$.

Let $W_{t} = \sum_{i=1}^N w_{t, i}$, and $q_{t, i} = \frac{w_{t, i}}{W_t}$, then
\begin{align}
\frac{W_{t+1}}{W_t} &= \sum_{i=1}^{N} \frac{w_{t+1, i}}{W_t} \nonumber \\
&= \sum_{i=1}^{N} q_{t, i} \exp(\gamma \hat{y}_{t, i} / K) \nonumber \\
&\le \sum_{i=1}^N q_{t, i} \left[ 1 + \frac{\gamma}{K} \hat{y}_{t, i} +
  (e-2)(\frac{\gamma}{K} \hat{y}_{t, i})^2 \right] \label{prof:exp4thm1eq1}\\
&\le 1 + \frac{\gamma}{K} \sum_{i=1}^N q_{t, i} \hat{y}_{t, i} +
  (e-2)\left( \frac{\gamma}{K} \right)^2 \sum_{i=1}^N q_{t, i} \hat{y}_{t, i}^2
  \nonumber \\
&\le \exp \left( \frac{\gamma}{K} \sum_{i=1}^N q_{t, i} \hat{y}_{t, i} + (e-2)\left(
  \frac{\gamma}{K} \right)^2 \sum_{i=1}^N q_{t, i} \hat{y}_{t, i}^2 \right) \label{prof:exp4thm1eq2}
\end{align}
Equation \ref{prof:exp4thm1eq1} is due to $e^x \le 1 + x + (e - 2) x^2$ for $x \le
1$, Equation \ref{prof:exp4thm1eq2} is due to $1+x \le e^x$. Taking logarithms and summing over t
\begin{align}
  \ln \frac{W_{T+1}}{W_1} \le \frac{\gamma}{K} \sum_{t=1}^T \sum_{i=1}^N q_{t, i} \hat{y}_{t,
  i} + (e-2)\left(\frac{\gamma}{K} \right)^2 \sum_{t=1}^T \sum_{i=1}^N q_{t, i}
  \hat{y}_{t, i}^2 \label{prof:exp4thm1eq3}
\end{align}
For any expert k
\begin{align*}
\ln \frac{W_{T+1}}{W_1} & \ge \ln \frac{w_{T+1, k}}{W_1} \\
&= \ln w_{1, k} + \sum_{t=1}^T (\frac{\gamma}{K} \hat{y}_{t, i}) - \ln W_1
  \\
&= \frac{\gamma}{K} \sum_{t=1}^T \hat{y}_{t, i} - \ln N
\end{align*}
Together with Equation\ \ref{prof:exp4thm1eq3} we get
\begin{align}
  \sum_{t=1}^T\sum_{i=1}^N q_{t, i} \hat{y}_{t, i} \ge \sum_{t=1}^T \hat{y}_{t,
  k} - \frac{K \ln N}{\gamma} - (e-2) \frac{r}{K} \sum_{t=1}^T \sum_{i=1}^N
  q_{t,i}\hat{y}_{t, i}^2 \label{prof:exp4thm1eq4}
\end{align}
Now we need to bound $\sum_{i=1}^N q_{t, i}\hat{y}_{t, i}$ and $\sum_{i=1}^N
q_{t, i}\hat{y}_{t, i}^2$. From the definition of $p_{t, i}$ we have
$\sum_{i=1}^N q_{t, i} \xi^i_{t, j} = \frac{p_{t, j} - \gamma/K}{1-\gamma}$, so
\begin{align*}
  \sum_{i=1}^N q_{t, i} \hat{y}_{t, i} &= \sum_{i=1}^N q_{t, i}\left(\sum_{j=1}^K
                                         \xi_{t, j}^i \hat{r}_{t, j}\right) \\
&= \sum_{j=1}^K \left(\sum_{i=1}^N q_{t, i} \xi_{t, j}^i  \right) \hat{r}_{t, j}
  \\
&= \sum_{j=1}^K \left( \frac{p_{t, j} - \gamma/K}{1-\gamma} \right) \hat{r}_{t,
  j} \\
&\le \frac{r_{t, a_t}}{1-\gamma}
\end{align*}
where $a_t$ is the arm pulled at time $t$.
\begin{align*}
\sum_{i=1}^N q_{t, i} \hat{y}_{t, i}^2 &= \sum_{i=1}^N q_{t, i} (\xi^i_{t, a_t}  \hat{r}_{t,
                                         a_t})^2 \\
& \le \sum_{i=1}^N q_{t, i} (\xi_{t, a_t}^i)^2 \hat{r}_{t, a_t}^2 \\
& \le \sum_{i=1}^N q_{t, i} \xi_{t, a_t}^i \hat{r}_{t, a_t}^2 \\
& \le \hat{r}_{t, a_t}^2 \frac{p_{t, a_t}}{1-\gamma} \\
& \le \frac{\hat{r}_{t, a_t}}{1-\gamma}
\end{align*}
Together with Equation \ref{prof:exp4thm1eq4} we have
\begin{align*}
  \sum_{t=1}^T r_{t, a_t} &\ge (1-\gamma) \sum_{t=1}^T \hat{y}_{t, k} - \frac{K \ln
  N}{\gamma}(1-\gamma) - (e-2)\frac{\gamma}{K}\sum_{t=1}^T \sum_{j=1}^K \hat{r}_{t, j} \\
  &\ge (1-\gamma) \sum_{t=1}^T \hat{y}_{t, k} - \frac{K \ln N}{\gamma} -
    (e-2)\frac{\gamma}{K}\sum_{t=1}^T \sum_{j=1}^K \hat{r}_{t, j} \\
\end{align*}
Taking expectation of both sides of the inequality we get
\begin{align}
\mathrm{E} \sum_{t=1}^T r_{t, a_t} & \ge (1-\gamma)\sum_{t=1}^T y_{t, k} -
                                      \frac{K\ln N}{\gamma} -
                                      (e-2)\frac{\gamma}{K} \sum_{t=1}^T
                                      \sum_{j=1}^K r_{t, j} \nonumber \\
& \ge (1-\gamma)\sum_{t=1}^T y_{t, k} - \frac{K\ln N}{\gamma} -
  (e-2) \gamma \sum_{t=1}^T \frac{1}{K}\sum_{j=1}^K r_{t, j}  \label{prof:exp4thm1eq5}
\end{align}
Since there is a uniform expert in the expert set, so $G_{max} \ge \sum_{t=1}^T
\frac{1}{K}\sum_{j=1}^K r_{t, j}$. Let $G_{exp4} = \sum_{t=1}^T r_{t,
a_t}$, then Equation\ \ref{prof:exp4thm1eq5} can be rewritten as
\begin{align*}
\mathrm{E} G_{exp4} \ge (1-\gamma)\sum_{t=1}^T y_{t, k} - \frac{K\ln N}{\gamma} -
  (e-2)\gamma G_{max}
\end{align*}
For any $k$. Let $k$ be the arm with the highest expected reward, then we have
\begin{align*}
\mathrm{E} G_{exp4} \ge (1-\gamma) G_{max} - \frac{K\ln N}{\gamma} -
  (e-2)\gamma G_{max} \\
G_{max} - \mathrm{E} G_{exp4} \le \frac{K \ln N}{\gamma} + (e-1)\gamma G_{max}
\end{align*}
We need to select a $\gamma$ such that the right-hand side of the above
inequality is minimized so that the regret bound is minimized. An additional
constrain is that $\gamma \le 1$. Taking the
derivative with respect to $\gamma$ and setting to 0, we get
\begin{align*}
\gamma^* = \min\left\{1, \sqrt{\frac{K\ln N}{(e-1)G_{max}}}\right\} \\
G_{max} - \mathrm{E} G_{exp4} \le 2.63 \sqrt{G_{max} K \ln N}
\end{align*}
Since $G_{max} \le T$, we have $R_T = O(\sqrt{TK\ln N})$. One important thing
to notice is that to get such regret bound it requires the knowledge of $T$,
the time horizon. Later we will introduce algorithms that does not require such
knowledge.
\end{proof}
\subsection{EXP4.P}
The unbiased estimator of the reward vector used by EXP4 has high variance due
to the increased range of the random variable
$r_{a_t}/p_{a_t}$ \citep{dudik2014doubly}, and the regret bound of EXP4,
$O(\sqrt{TK\ln N})$, is hold only
with expectation. EXP4.P \citep{beygelzimer2010contextual} improves this result and
achieves the same regret with high probability.
To do this, EXP4.P combines the idea of both UCB \citep{auer2002finite} and
EXP4. It computes the confidence interval of the reward vector estimator and
hence bound the cumulative reward of each expert with high probability, then it
designs an strategy to weight each expert.

Similar to the EXP4 algorithm setting, there are K arms $\{1, 2, ..., K\}$ and N
experts $\{1, 2, ..., N\}$. At time $t \in \{1, ..., T\}$, the world reveals
context $x_t$, and each expert $i$ outputs
an advice vector $\ve{\xi}^i_t$ representing its recommendations on each
arm. The agent then selects an arm $a_t$ based on the advice, and an adversary
chooses a reward vector $\ve{r}_t$. Finally the world reveals the reward of the chosen
arm $r_{t, a_t}$. Let $G_i$ be the expected cumulative reward of expert $i$:
\begin{align*}
  G_i = \sum_{t=1}^T \ve{\xi}^i_t \cdot \ve{r}
\end{align*}
let $p_j$ be the algorithm's probability of pulling arm $j$, and let $\ve{\hat{r}}$
be the estimated reward vector, where
\[
\hat{r}_j =
\begin{cases}
  r_j / p_j & \text{if\quad } j=a_t \\
  0 & \text{if\quad } j \ne a_t
\end{cases}
\]
let $\hat{G}_i$ be the estimated expected cumulative reward of expert $i$:
\begin{align*}
\hat{G}_i = \sum_{t=1}^T \ve{\xi}^i_t \cdot \ve{\hat{r}}
\end{align*}
let $G_{exp4.p}$ be the estimated cumulative reward of the algorithm:
\begin{align*}
G_{exp4.p} = \sum_{t=1}^T r_{a_t}
\end{align*}
then the expected regret of the algorithm is 
\begin{align*}
R_T = \max_i G_i - \mathrm{E}G_{exp4.p}
\end{align*}
However, we are interested in regret bound which hold with arbitrarily high
probability. The regret is bounded by $\epsilon$
with probability $1-\delta$ if 
\begin{align*}
P\left( \left(\max_i G_i - G_{exp4.p}\right) > \epsilon \right) \le \delta
\end{align*}
We need to bound $G_i - \hat{G}_i$ with high probability so that we can bound the
regret with high probability. To do that, we need to use the following theorem:

\begin{theorem}
\label{thm:exp4pbasethm}
Let $X_1, ..., X_T$ be a sequence of real-valued random variables. Suppose that $X_t \le
R$ and $\mathrm{E}(X_t) = 0$. Define the random variables
\begin{align*}
  S = \sum_{t=1}^T X_t, \text{\quad} V=\sum_{t=1}^T \mathrm{E}(X_t^2)
\end{align*}
then for any $\delta$, with probability $1-\delta$, we have
\begin{align*}
S \le
\begin{cases}
\sqrt{(e-2)\ln(1/\delta)}\left( \frac{V}{\sqrt{V'}} + \sqrt{V'} \right) &
\text{if\quad } V' \in \left[ \frac{R^2\ln(1/\delta)}{e-2}, \infty \right)\\
R\ln (1/\delta) + (e-2)\frac{V}{R} & \text{if\quad } V' \in \left[ 0, \frac{R^2 \ln(1/\delta)}{e-2} \right]
\end{cases}
\end{align*}
\end{theorem}

\noindent To bound $\hat{G}_i$, let $X_t = \ve{\xi}^i_t \cdot \ve{r_t} - \ve{\xi}^i_t
\cdot \ve{\hat{r}_t}$, so $\mathrm{E}(X_t) = 0$, $R = 1$ and
\begin{align*}
  \mathrm{E}(X_t^2) &\le \mathrm{E} (\ve{\xi}^i_t \cdot \ve{\hat{r}_t})^2 \\
&= \sum_{j=1}^K p_{t, j} \left( \xi_{t, j}^i \cdot \frac{r_{t, j}}{p_{t, j}} \right)^2 \\
  &\le \sum_{j=1}^K \frac{\xi^i_{t, j}}{p_{t, j}} \\
&\overset{\text{def}}{=} \hat{v}_{t, i}
\end{align*}
The above proof used the fact that $r_{t, j} \le 1$. Let $V' = KT$, assume $\ln
(N / \delta) \le KT$, and use $\delta/N$ instead of $\delta$, we can apply Theorem
\ref{thm:exp4pbasethm} to get
\begin{align*}
 P\left(G_i - \hat{G}_i \ge \sqrt{(e-2) \ln \frac{N}{\delta}}\left(\frac{\sum_{t=1}^T \hat{v}_{t,
  i}}{\sqrt{KT}} + \sqrt{KT}\right) \right) \le \frac{\delta}{N}
\end{align*}
Apply union bound we get:
\begin{theorem}
  \label{thm:exp4pcithm}
  Assume $\ln (N / \delta) \le KT$, and define $\hat{\sigma}_i = \sqrt{KT} +
  \frac{1}{\sqrt{KT}}\sum_{t=1}^T \hat{v}_{t, i}$, we have that with probability $1-\delta$
\begin{align*}
 \sup_{i} (G_i - \hat{G}_i) \le \sqrt{\ln \frac{N}{\delta}} \hat{\sigma}_i
\end{align*}
\end{theorem}
The confidence interval we get from Theorem \ref{thm:exp4pcithm} is used to
construct EXP4.P algorithm. The detail of the algorithm is described in Algorithm \ref{algo:exp4.p}. We
can see that EXP4.P is very similar to EXP4 algorithm, except that when
updating $w_{t, i}$, instead of using estimated reward, we use the upper
confidence bound of the estimated reward. 
\begin{algorithm}[!htp]
\caption{EXP4.P}
\label{algo:exp4.p}
\begin{algorithmic}
  \Require $\delta > 0$
  \State Define $p_{min}=\sqrt{\frac{\ln N}{KT}}$, set $w_{1, i} = 1$ for $i=1,...,N$.
\For{$t=1, 2, ... T$}
\State Get expert advice vectors $\{\ve{\xi^1_t}, \ve{\xi^2_t}, ..., \ve{\xi^N_t}\}$.
\For{$j=1, 2, ..., K$}
\begin{align*}
  p_{t, j} = (1-Kp_{min})\sum_{i=1}^N  \frac{w_{t, i}\xi^i_{t, j}}{\sum_{i=1}^N w_{t,
  i}} + p_{min}
\end{align*}
\EndFor
\State Draw action $a_t$ according to $p_t$ and receive reward $r_{a_t}$.
  \For{$j=1, ..., K$} 
  \begin{align*}
    \hat{r}_{t, j} = \frac{r_{t, j}}{p_{t, j}} \mathds{I} (j = a_t)
  \end{align*}
  \EndFor
  \For{$i=1, ..., N$}
\begin{align*} 
  \hat{y}_{t, i} &= \ve{\xi}^i_t \cdot \ve{\hat{r}}_{t} \\
  \hat{v}_{t, i} &= \sum_{j=1}^K \xi^i_{t, j}/p_{t, j} \\
  w_{t+1, i} &= w_{t, i} \exp\left(\frac{p_{min}}{2}\left(\hat{y}_{t, i} + \hat{v}_{t, i}
               \sqrt{\frac{\ln (N/\delta)}{KT}} \right)\right)
  \end{align*}
  \EndFor
\EndFor
\end{algorithmic}
\end{algorithm}

\begin{theorem}
  Assume that $\ln (N/\delta) \le KT$, and the set of experts includes a uniform
  expert which selects an arm uniformly at randomly at each time. Then with
  probability $1-\delta$
\begin{align*}
R_T = \max_i G_i - G_{exp4.p} \le 6\sqrt{KT\ln (N/\delta)}
\end{align*}
\end{theorem}
\begin{proof}
The proof is similar to the proof of regret bound of EXP4. Basically, we want to
bound $G_{exp4.p} = \sum_{t=1}^T r_{a_t}$, and 
since we can bound $\max_i G_i$ with high probability, we then get the regret of
EXP4.P with high probability.

\noindent Let $q_{t, i} = \frac{w_{t, i}}{\sum_i w_{t, i}}$, $\gamma =
\sqrt{\frac{K\ln N}{T}}$, and $\hat{U} = \max_i (\hat{G}_i +
\hat{\sigma}_i\sqrt{\ln (N/\delta)})$. We need the following inequalities
\begin{align*}
  \hat{v}_{t, i} \le 1/p_{min} \\
  \sum_{i=1}^N q_{t, i} \hat{v}_{t, i} \le \frac{K}{1-\gamma}
\end{align*}
To see why this is true:
\begin{align*}
\sum_{i=1}^N q_{t, i} \hat{v}_{t, i} &= \sum_{i=1}^N q_{t, i} \sum_{j=1}^K
                                       \frac{\xi^i_{t, j}}{p_{t, j}} \\
&= \sum_{j=1}^K \frac{1}{p_{t, j}} \sum_{i=1}^N q_{t, i} \xi^i_{t, j} \\
&\le \sum_{j=1}^K \frac{1}{1-\gamma} \\
&=\frac{K}{1-\gamma}
\end{align*}
We also need the following two inequalities, which has been proved in Section \ref{sec:exp4}.
\begin{align*}
  \sum_{i=1}^N q_{t, i}\hat{y}_{t, i} \le \frac{r_{t, a_t}}{1-\gamma} \\
  \sum_{i=1}^N q_{t, i}\hat{y}_{t, i}^2 \le \frac{\hat{r}_{t, a_t}}{1-\gamma}
\end{align*}

\noindent Let $b=\frac{p_{min}}{2}$ and $c=\frac{p_{min}\sqrt{\ln
    (N/\delta)}}{2\sqrt{KT}}$, then 
\begin{align*}
\frac{W_{t+1}}{W_t} &= \sum_{i=1}^N \frac{w_{t+1, i}}{W_t} \\
&= \sum_{i=1}^N q_{t, i} \exp(b\hat{y}_{t, i} + c\hat{v}_{t, i})
\end{align*}
Since $e^a \le 1 + a + (e-2)a^2$ for $a \le 1$ and $e-2 \le 1$, we have
\begin{align*}
\frac{W_{t+1}}{W_t} &\le \sum_{i=1}^N q_{t, i}(1+b\hat{y}_{t, i} + c\hat{v}_{t,
  i}) + \sum_{i=1}^N q_{t, i} (2b^2\hat{y}_{t, i}^2 + 2c^2\hat{v}_{t, i}^2) \\
&= 1 + b\sum_{i=1}^N q_{t, i}\hat{y}_{t, i} + c\sum_{i=1}^N q_{t,i}\hat{v}_{t,i}
  + 2b^2\sum_{i=1}^N q_{t, i}\hat{y}_{t,i}^2 + 2c^2\sum_{i=1}^N q_{t,
  i}\hat{v}_{t, i}^2 \\
&\le 1 + b\frac{r_{t, a_t}}{1-\gamma} + c\frac{K}{1-\gamma} +
  2b^2\frac{\hat{r}_{t, a_t}}{1-\gamma} + 2c^2 \sqrt{\frac{KT}{\ln N}} \frac{K}{1-\gamma}
\end{align*}
Take logarithms on both side, sum over T and make use of the fact that $\ln(1+x)
\le x$ we have
\begin{align*}
\ln \left(\frac{W_{T+1}}{W_1}\right) \le \frac{b}{1-\gamma}\sum_{t=1}^T r_{t,
  a_t} + c\frac{KT}{1-\gamma} + \frac{2b^2}{1-\gamma}\sum_{t=1}^T \hat{r}_{t,
  a_t} + 2c^2 \sqrt{\frac{KT}{\ln N}}\frac{KT}{1-\gamma}
\end{align*}
Let $\hat{G}_{uniform}$ be the estimated cumulative reward of the uniform
expert, then 
\begin{align*}
\hat{G}_{uniform} &= \sum_{t=1}^T \sum_{j=1}^K \frac{1}{K} \hat{r}_{j} \\
&= \sum_{t=1}^T \frac{1}{K} \hat{r}_{t, a_t}
\end{align*}
So
\begin{align*}
\ln \left(\frac{W_{T+1}}{W_1}\right) &\le \frac{b}{1-\gamma}\sum_{t=1}^T r_{t,
  a_t} + c\frac{KT}{1-\gamma} + \frac{2b^2}{1-\gamma}\sum_{t=1}^T K
  \hat{G}_{uniform} + 2c^2 \sqrt{\frac{KT}{\ln N}}\frac{KT}{1-\gamma} \\
&\le \frac{b}{1-\gamma}\sum_{t=1}^T r_{t,
  a_t} + c\frac{KT}{1-\gamma} + \frac{2b^2}{1-\gamma}\sum_{t=1}^T K \hat{U} + 2c^2 \sqrt{\frac{KT}{\ln N}}\frac{KT}{1-\gamma}
\end{align*}
Also 
\begin{align*}
  \ln (W_{T+1}) &\ge \max_i (\ln w_{T+1, i}) \\
&= \max_i \left(b\hat{G}_i + c\sum_{t=1}^T \hat{v}_{t, i}\right) \\
&= b\hat{U} - b\sqrt{KT\ln (N/\delta)}
\end{align*}
So
\begin{align*}
b\hat{U} - b\sqrt{KT\ln (N/\delta)}  - \ln N \le \frac{b}{1-\gamma} G_{exp4.p} +
  c\frac{KT}{1-\gamma} + \frac{2b^2}{1-\gamma}\sum_{t=1}^T K \hat{U} + 2c^2
  \sqrt{\frac{KT}{\ln N}}\frac{KT}{1-\gamma} \\
G_{exp4.p} \ge \left( 1-2\sqrt{\frac{K\ln N}{T}} \right)\hat{U} - \ln(N/\delta) -
  2\sqrt{KT\ln N} - \sqrt{KT\ln(N/\delta)}
\end{align*}
We already know from Theorem \ref{thm:exp4pbasethm} that $\max_i G_i \le
\hat{U}$ with probability $1-\delta$, and also $\max_i G_i \le T$, so with
probability $1-\delta$
\begin{align*}
G_{exp4.p} &\ge \max_i G_i - 2\sqrt{\frac{K\ln N}{T}}T - \ln(N/\delta) -
  \sqrt{KT\ln N} - 2\sqrt{KT\ln(N/\delta)} \\
&\ge \max_i G_i - 6\sqrt{KT\ln (N/\delta)}
\end{align*}
\end{proof}
\subsection{Infinite Many Experts}
Sometimes we have infinite number of experts in the expert set $\Pi$. For example, an expert could be
a d-dimensional vector $\beta \in \mathrm{R}^d$, and the predictive reward could
be $\beta^\top x$ for some context $x$. Neither EXP4 nor EXP4.P are
able to handle infinite experts.

A possible solution is to construct a finite approximation $\hat{\Pi}$
to $\Pi$, and then use EXP4 or EXP4.P on $\hat{\Pi}$ \citep{bartlett2014cs294,
beygelzimer2010contextual}. Suppose for every expert $\pi \in \Pi$ there is a
$\hat{\pi} \in \hat{\Pi}$ with
\begin{align*}
P(\pi(x_t) \ne \hat{\pi}(x_t)) \le \epsilon
\end{align*}
where $x_t$ is the context and $\pi(x_t)$ is the chosen arm. Then the reward
$r \in [0, 1]$ satisfy
\begin{align*}
\mathrm{E} \left|r_{\pi(x_t)} - r_{\hat{\pi}(x_t)}\right| \le \epsilon
\end{align*}
We compete with the best expert in $\Pi$, the regret is
\begin{align*}
R_T(\Pi) = \sup_{\pi \in \Pi} \mathrm{E} \sum_{t=1}^T r_{\pi(x_t)} -
  \mathrm{E} \sum_{t=1}^T r_{a_t}
\end{align*}
And we can bound $R_T(\Pi)$ with $R_T(\hat{\Pi})$:
\begin{align*}
R_T(\Pi) &= \sup_{\pi \in \Pi} \mathrm{E} \sum_{t=1}^T r_{\pi(x_t)} -
  \sup_{\hat{\pi} \in \hat{\Pi}} \mathrm{E} \sum_{t=1}^T r_{\hat{\pi}(x_t)} +
  \sup_{\hat{\pi} \in \hat{\Pi}} \mathrm{E} \sum_{t=1}^T r_{\hat{\pi}(x_t)} -
  \mathrm{E} \sum_{t=1}^T r_{a_t} \\
 &= \sup_{\pi \in \Pi} \inf_{\hat{\pi} \in \hat{\Pi}} \mathrm{E} \sum_{t=1}^T \left(r_{\pi(x_t)} - r_{\hat{\pi}(x_t)}\right) +
  \sup_{\hat{\pi} \in \hat{\Pi}} \mathrm{E} \sum_{t=1}^T r_{\hat{\pi}(x_t)} -
  \mathrm{E} \sum_{t=1}^T r_{a_t} \\
  &\le T\epsilon + R_T(\hat{\Pi})
\end{align*}
There are many ways to construct such $\hat{\Pi}$. Here we talk about an
algorithm called VE \citep{beygelzimer2010contextual}. The idea is to choose an
arm uniformly at random for the first $\tau$ rounds, then we get $\tau$ contexts
$x_1, ..., x_\tau$. Given an expert $\pi \in \Pi$,
we can get a sequence of prediction $\{\pi(x_1), ..., \pi(x_\tau)\}$. Such
sequence is enumerable, so we can construct $\hat{\Pi}$ containing one
representative $\hat{\pi}$ for each sequence $\{\hat{\pi}(x_1), ...,
\hat{\pi}(x_\tau)\}$. Then we apply EXP4/EXP4.P on $\hat{\Pi}$. VE is shown in Algorithm \ref{algo:ve}.

\begin{algorithm}[!htp]
\caption{VE}
\label{algo:ve}
\begin{algorithmic}
  \Require $\tau$
  \For{$t=1, 2, ... \tau$}
  \State Receive context $x_t$
  \State Choose arm uniformly at random
  \EndFor
  \State Construct $\hat{\Pi}$ based on $x_1, ..., x_\tau$
  \For{$t=\tau+1, ..., T$}
  \State Apply EXP4/EXP4.P
  \EndFor
\end{algorithmic}
\end{algorithm}
\begin{theorem}
For all policy sets $\Pi$ with VC dimension $d$,
$\tau=\sqrt{T\left(2d\ln\frac{eT}{d}+ \ln\frac{2}{\delta}\right)}$, with probability
$1-\delta$
\begin{align*}
R_T \le 9\sqrt{2T\left(d\ln\frac{eT}{d}+\ln\frac{2}{\delta}\right)}
\end{align*}
\end{theorem}
\begin{proof}
  Given $\pi \in \Pi$ and corresponding $\hat{\pi} \in \hat{\Pi}$
\begin{align}
G_\pi = G_{\hat{\pi}} + \sum_{t=\tau+1}^T \mathds{I}(\pi(x_t) \ne
  \hat{\pi}(x_t)) \label{eq:veeq2}
\end{align}
We need to measure the expected disagreements of $\pi$ and $\hat{\pi}$
after time $\tau$. Suppose the total disagreements within time $T$ is $n$, then
if we randomly pick $\tau$ contexts, the probability that $\pi$ and $\hat{\pi}$
produce the same sequence is
\begin{align*}
P\left(\forall t \in [1, \tau], \pi(x_t) = \hat{\pi}(x_t)\right) &=
  \left(1-\frac{n}{T}\right)\left(1-\frac{n}{T-1}\right)...\left(1-\frac{n}{T-\tau+1}\right) \nonumber \\
  &\le \left(1-\frac{n}{T}\right)^\tau \nonumber \\
  &\le e^{-\frac{n\tau}{T}} \label{eq:veeq1}
\end{align*}
From Sauer's
lemma we have that $|\hat{\Pi}| \le (\frac{e\tau}{d})^d$ for all $\tau > d$ and
the number of unique sequences produced by all $\pi \in \Pi$ is less than
$(\frac{e\tau}{d})^d$ for all $\tau > d$. For a $\pi \in \Pi$ and corresponding $\hat{\pi} \in
\hat{\Pi}$, we have
\begin{align*}
&P\left(\sum_{t=\tau+1}^T \mathds{I}(\pi(x_t) \ne \hat{\pi}(x_t)) > n \right) \\
&\le P\left( \exists \pi', \pi'' : \sum_{t=\tau+1}^T \mathds{I}(\pi'(x_t) \ne
    \pi''(x_t)) > n \text{ and } \forall t \in [1, \tau], \pi'(x_t) =
  \pi''(x_t) \right) \\
&\le |\Pi|^2 e^{-\frac{n\tau}{T}} \\
&\le \left(\frac{e\tau}{d}\right)^{2d} e^{-\frac{n\tau}{T}}
\end{align*}
Set the right-hand side to $\frac{\delta}{2}$ and we get:
\begin{align*}
n \ge \frac{T}{\tau}\left( 2d\ln\frac{eT}{d} + \ln\frac{2}{\delta} \right)
\end{align*}
Together with Equation (\ref{eq:veeq2}), we get with probability $1-\frac{\delta}{2}$
\begin{align*}
  G_{\max(\hat{\Pi})} \ge G_{\max(\Pi)} - \frac{T}{\tau}\left(
  2d\ln\frac{eT}{d} + \ln\frac{2}{\delta} \right)
\end{align*}
Now we need to bound $G_{\max{(\hat{\Pi}})}$. From Sauer's lemma we have that
$|\hat{\Pi}| \le (\frac{e\tau}{d})^d$ for all $\tau > d$, so we can directly
apply EXP4.P's bound. With probability $1-\frac{\delta}{2}$
\begin{align*}
G_{exp4.p}(\hat{\Pi}, T-\tau) \ge G_{\max{(\hat{\Pi}})} - 6\sqrt{2(T-\tau)(d\ln(\frac{e\tau}{d})+\ln(\frac{2}{\delta}))}
\end{align*}
Finally, we get the bound on $G_{VE}$
\begin{align*}
G_{VE} \ge G_{\max(\Pi)} - \tau - \frac{T}{\tau}\left( 2d\ln\frac{eT}{d} +
  \ln\frac{2}{\delta} \right) - 6\sqrt{2(T-\tau)(d\ln(\frac{e\tau}{d})+\ln(\frac{2}{\delta}))}
\end{align*}
Setting $\tau = \sqrt{T(2d\ln\frac{eT}{d}+\ln\frac{2}{\delta})}$ we get
\begin{align*}
  G_{VE} \ge G_{\max(\Pi)} - 9 \sqrt{2T(d\ln\frac{eT}{d}+\ln\frac{2}{\delta})}
\end{align*}
\end{proof}

\section{Conclusion}
The nature of contextual bandits makes it suitable for many machine learning
applications such as user modeling, Internet advertising, search engine,
experiments optimization etc., and there has been a growing interests in this
area. One topic we haven't covered is the offline evaluation in contextual
bandits. This is tricky since the policy evaluated is different from the
policy that generating the data, so the arm proposed offline does not
necessary match the one pulled online. \citet{li2011unbiased} proposed an
unbiased offline evaluation method assuming that the logging policy selects arm
uniformly at random. \citet{strehl2010learning} proposed an methods that will estimate the
probability of the logging policy selecting each arm, and then adopt inverse
propensity score(IPS) to evaluation new policy, \citet{langford2011doubly}
proposed an method that combines the direct method and IPS to improve accuracy
and reduce variance.

Finally, note that regret bound is not the only criteria for bandits
algorithm. First of all, the bounds we talked about in this survey are
problem-independent bounds, and there are problem-dependent bounds. For example,
\citet{langford2008epoch} proved that although the Epoch-Greedy's
problem-independent bound is not optimal, it can achieve a $O(\ln T)$ problem-dependent bound; Second, different bandits algorithms have their own
different assumptions (stochastic/adversarial, linearity, number of policies, Bayesian
etc.), so when choosing which one to use, we need to choose the one
matches our assumptions.
\bibliography{ref}
\end{document}